\documentclass[pdflatex,sn-mathphys-ay]{sn-jnl}

\usepackage{amssymb,amsmath, amsfonts,bm,mathtools,amsthm}
\usepackage{graphicx,epsfig,times,tikz,subfig,ifxetex,soul,hyperref,float}
\usepackage{hhline}
\usepackage[shortlabels]{enumitem}

\usepackage{booktabs}

\usepackage{caption}

\newtheorem{theorem}{Theorem}[section]
\newtheorem{definition}{Definition}[section]
\newtheorem{lemma}{Lemma}[section]
\newtheorem{corollary}{Corollary}[section]
\newtheorem{proposition}{Proposition}[section]
\newtheorem{example}{Example}[section]

\newcommand\IZD[1]{\ensuremath{#1|_0}\xspace}
\newcommand\IZU[1]{\ensuremath{#1|^0}\xspace}  
\newcommand\IZZ[1]{\ensuremath{#1|_{0}^0}\xspace}

\usepackage{xspace}

\newcommand{\Ione}{{\bf(I1)}\xspace}
\newcommand{\Itwo}{{\bf(I2)}\xspace}
\newcommand{\Ithree}{{\bf(I3)}\xspace}

\newcommand{\NP}{{\bf(NP)}\xspace}

\newcommand{\IP}{{\bf(IP)}\xspace}
\newcommand{\OP}{{\bf(OP)}\xspace}
\newcommand{\CB}{{\bf(CB)}\xspace}

\newcommand{\LF}{{\bf(LF)}\xspace}
\newcommand{\LT}{{\bf(LT)}\xspace}
\newcommand{\CPN}{{\bf(CP(N))}\xspace}
\newcommand{\LCPN}{{\bf(L$\mhyphen$CP(N))}\xspace}
\newcommand{\RCPN}{{\bf(R$\mhyphen$CP(N))}\xspace}

\newcommand{\PIT}{{\bf(PI$_{\bm{T}}$)}\xspace}

\newcommand{\IRC}{\ensuremath{I_{\bm{RC}}}\xspace}

\newcommand{\ILK}{\ensuremath{I_{\bm{LK}}}\xspace}

\newcommand{\IKD}{\ensuremath{I_{\bm{KD}}}\xspace}

\usepackage{cellspace} 
\usepackage{adjustbox}
\usepackage{multicol}
\usepackage{multirow}
\usepackage{array}
\newcolumntype{M}[1]{>{\centering\arraybackslash}m{#1}}
\usepackage{pifont}
\newcommand{\cmark}{\ding{51}}%
\def\mystrut(#1,#2){\vrule height #1 depth #2 width 0pt}
\newcolumntype{C}[1]{%
	>{\mystrut(3ex,2ex)\centering}%
	p{#1}%
	<{}}

\mathchardef\mhyphen="2D

\newcommand{\I}{\ensuremath{(I_1,\ldots,I_n)_{F,\textbf{c}_1,\textbf{c}_2}}\xspace}
\newcommand{\Is}{\ensuremath{\bm{I}_{F,\textbf{c}_1,\textbf{c}_2}}\xspace}

\usepackage[normalem]{ulem}
\makeatletter
\renewcommand*\l@figure{\@dottedtocline{1}{1em}{3em}}
\renewcommand*\l@table{\@dottedtocline{1}{1em}{3em}}

\usepackage{stackengine}

\begin{document}

\title[Article Title]{A global view of diverse construction methods of fuzzy implication functions rooted on $F$-chains}

\author*[1]{\fnm{Raquel} \sur{Fernandez-Peralta}}\email{raquel.fernandez@mat.savba.sk}

\author[2,3,4]{\fnm{Juan} \sur{Vicente Riera}}\email{jvicente.riera@uib.es}
\equalcont{These authors contributed equally to this work.}

\affil*[1]{\orgdiv{Mathematical Institute}, \orgname{Slovak Academy of Sciences}, \orgaddress{\street{\v{S}tef\'anikova~49}, \city{Bratislava}, \postcode{814 73}, \country{Slovakia}}}

\affil[2]{\orgdiv{Soft Computing, Image Processing and Aggregation (SCOPIA) research group. Dept. of Mathematics and Computer Science}, \orgname{University of the Balearic Islands}, \orgaddress{\street{Ctra. de Valldemossa, Km.7.5}, \city{Palma}, \postcode{07122}, \country{Spain}}}

\affil[3]{\orgname{Balearic Islands Health Research Institute (IdISBa)}, \orgaddress{\city{Palma}, \postcode{07010}, \country{Spain}}}
\affil[4]{\orgname{Artificial Intelligence Research Institute of the Balearic Islands (IAIB)}, \orgaddress{\city{Palma}, \postcode{07122}, \country{Spain}}}


\abstract{Fuzzy implication functions are one of the most important operators used in the fuzzy logic framework. While their flexible definition allows for diverse families with distinct properties, this variety needs a deeper theoretical understanding of their structural relationships. In this work, we focus on the study of construction methods, which employ different techniques to generate new fuzzy implication functions from existing ones. Particularly, we generalize the $F$-chain-based construction, recently introduced by Mesiar et al. to extend a method for constructing aggregation functions to the context of fuzzy implication functions.  Our generalization employs collections of fuzzy implication functions rather than single ones, and uses two different increasing functions instead of a unique $F$-chain. We analyze property preservation under this construction and establish sufficient conditions. Furthermore, we demonstrate that our generalized $F$-chain-based construction is a unifying framework for several existing methods. In particular, we show that various construction techniques, such as contraposition, aggregation, and generalized vertical/horizontal threshold methods, can be reformulated within our approach. This reveals structural similarities between seemingly distinct construction strategies and provides a cohesive perspective on fuzzy implication construction methods.}

\keywords{Fuzzy implication function, construction method, $F$-chain, aggregation function}

\maketitle

\section{Introduction}

Fuzzy implication functions are one of the main operators in fuzzy logic, where they play the role of generalizing the classical logical conditional to more general domains than $\{0,1\}$, usually the unit interval, although others have also been considered \cite{Goguen1967,Munar2023}. Since the definition of these operators only imposes monotonicity and boundary conditions, many families that fulfill these conditions can be defined. Indeed, the construction, representation, and characterization of different families of fuzzy implication functions and the study of desirable additional properties is an active research area. Further, these operators play a significant role in various applications such as approximate reasoning \cite{Jayaram2008B}, image processing \cite{Gonzalez2018}, fuzzy control \cite{Mendel2023}, data mining and knowledge extraction \cite{Fernandez-Peralta2025B,fernandez2023subgroup,Nanavati2024}; among others.

Up to now, a huge amount of families of fuzzy implication functions have been introduced so far. Although each family has its own motivation and special behavior (see \cite{Fernandez-Peralta2025} for a comprehensive overview), the boost of defining many families is associated with the fact that the need for different operators depending on the application has been widely discussed \cite{Trillas2008}. Indeed, it may not be straightforward to find an operator that fulfills the requirements of the particular problem \cite{Baczynski2020B,Fernandez-Peralta2021,Jayaram2008,Mis2025}. However, recent discussions \cite{Massanet2024} reveal the other side of the coin and stress that having too many families without a sufficient theoretical understanding of their relationships, clear practical contributions, or well-defined applicability may lead to confusion among practitioners and unnecessary overlap in the field. This highlights the importance of advancing theoretical studies that focus on characterizing, classifying, and unifying the structure of different classes of fuzzy implication functions. 

Typically, families are categorized into four main classes based on their generating mechanisms: (i) those derived from other aggregation functions, (ii) those constructed using unary functions, (iii) those defined by a fixed closed-form expression, and (iv) those obtained from other fuzzy implication functions. In this particular paper, we are interested in the latter type, which are also referred to as construction methods. Construction methods of fuzzy implication functions can be used not only to generate new operators, but they are also related to important theoretical aspects of the structure of these operators \cite{Vemuri2014}. To name a few strategies, there exist construction methods based on contrapositivisation \cite{Jayaram2006}, aggregation \cite{Reiser2013}, conjugation \cite{Baczynski2008} or ordinal sums \cite{Zhou2021B}. In this paper, we are specifically interested in the $F$-chain-based construction which was introduced in \cite{Mesiar2019} under the motivation of applying the method for constructing aggregation functions based on bijective chains presented in \cite{Jin2019} but to the case of fuzzy implication functions. 


In this work, we leverage the $F$-chain-based construction to define a new method for a family of fuzzy implication functions that is based on the use of an aggregation function and two input transformations. Although we stress in this paper that the $F$-chain condition is not necessary to ensure that the output is a fuzzy implication function, we highlight that it is required for the preservation of some desirable properties. As usual for construction methods, we study the preservation of several properties. In fact, we prove it is possible to obtain sufficient conditions to ensure the preservation of any of the considered properties, which is not generally true for less general methods. Furthermore, we show that the 
generalized $F$-chain-based construction method is closely linked to several other existing methods. Specifically, we prove that multiple distinct construction methods can be reformulated using a particular aggregation function and input transformations. Consequently, this approach not only serves as a valid new construction method but also provides a unifying framework that offers a broader perspective on various construction techniques.


The paper is structured as follows, in Section \ref{section:preliminaries} we recall some basic results on aggregation functions and fuzzy implication functions; in Section \ref{section:properties} we introduce the generalized $F$-chain-based construction method and we study the preservation of several additional properties; in Section \ref{section:construction_methods} we prove that the generalized $F$-chain-based construction method can be seen as the generalization of several other methods in the literature; finally the paper ends in Section \ref{section:conclusions} with some conclusions and future work.

\section{Preliminaries}\label{section:preliminaries}

In this section, we only introduce the results and definitions related to aggregation functions and fuzzy implication functions that are crucial for the well-comprehension of the paper. For more information about this topic we refer the reader to the books \cite{Baczynski2008,Beliakov2010,Calvo2002,Grabisch2009,Klement2000}.

First of all, we define aggregation functions and related terms, along with relevant examples.

\begin{definition}\label{def:aggregation_function}
    Let $n \in \mathbb{N}$. A monotone function $A:[0,1]^n \to [0,1]$ is called an \emph{$n$-ary aggregation function} whenever it satisfies the boundary conditions $A(0, \dots, 0) = 0$ and $A(1, \dots, 1) =1$. The class of all $n$-ary aggregation functions on $[0,1]$ will be denoted by $\mathcal{A}_n$.
\end{definition}

\begin{example}
Let $n \in \mathbb{N}$. Some examples of $n$-ary aggregation functions are the following:
\begin{itemize}
\item \emph{Maximum}: $\displaystyle F(x_1,\dots,x_n) = \max \{x_1,\dots,x_n\}$.
\item \emph{Minimum}: $\displaystyle F(x_1,\dots,x_n) = \min \{x_1,\dots,x_n\}$.
\item \emph{Product}: $\displaystyle F(x_1,\dots,x_n) = \prod_{i=1}^n x_i$.
\item \emph{Weighted arithmetic mean}: $F(x_1,\dots,x_n) = \sum_{i=1}^n w_ix_i$, were $\bm{w} = (w_1,\dots,w_n) \in [0,1]^n$ is such that $\sum_{i=1}^n w_i=1$.
\end{itemize}
\end{example}

\begin{definition}[\cite{GARCIALAPRESTA2008}]
Let $n \in \mathbb{N}$ and $F$ an $n$-ary aggregation function.
\begin{itemize}
\item An element $a \in [0,1]$ is said to be an \emph{annihilator} if $F(x_1,\dots,x_n)=a$ in case there exists $i \in \{1,\cdots,n\}$ such that $x_i=a$.
\item An element $c \in [0,1]$ is said to be \emph{idempotent} if $F(c,\dots,c)=c$.
\item An element $e \in [0,1]$ is said to be \emph{neutral} if $F(x_1,\dots,x_n)=x_i$ for any $x_i \in [0,1]$ in case that $x_j=e$ for all $j \in \{1,\dots,n\}$, $j \not = i$.
\item If $N$ is a fuzzy negation, then $F$ is \textit{self $N$-dual} if for all $(x_1,\dots,x_n)=\bm{x} \in [0,1]^n$ we have $F(N(x_1),\dots, N(x_n)) = N(F(\bm{x}))$.
\end{itemize}
\end{definition}


\begin{definition}
Let $n \in \mathbb{N}$ and $F$ an aggregation function, then a point $(x_1,\dots,x_n) \in [0,1]^n$ is called a
\begin{itemize}
\item \emph{unit multiplier} of $F$ if and only if $F(x_1,\dots,x_n)=1$ and there exists an $i \in \{1,\dots,n\}$ such that $x_i \not =1$.
\item \emph{zero multiplier} of $F$ if and only if $F(x_1,\dots,x_n)=0$ and there exists an $i \in \{1,\dots,n\}$ such that $x_i \not =0$.
\end{itemize}
\end{definition}


Next, we provide the definition of fuzzy implication function.

\begin{definition}[\cite{Baczynski2008,Fodor1994}]\label{def:implication}
	A binary operator $I:[0,1]^2 \to [0,1]$ is said to be a \emph{fuzzy implication function} if it satisfies:
	\begin{description}
		\item[\Ione]  $I(x,z)\geq I(y,z)\ $  when  $\ x\leq y$, for all $x,y,z\in[0,1]$. \hfill (Left Antitonicity)
		\item[\Itwo]  $I(x,y)\leq I(x,z)\ $  when  $\ y\leq z$, for all $x,y,z\in[0,1]$. \hfill (Right Isotonicity)
		\item[\Ithree]  $I(0,0)=I(1,1)=1$ and $I(1,0)=0$. \hfill (Boundary Condition)
	\end{description}
 The class of all fuzzy implication functions will be denoted by $\mathcal{I}$.
\end{definition}

From Definition~\ref{def:implication} it directly follows that $I(0,x) = I(x,1) = 1$ for all $x \in [0,1]$. However, the values $I(x,0)$ and $I(1,x)$ remain unfixed. In particular, the values $N_I(x)=I(x,0)$ define the \emph{natural negation} of a fuzzy implication function.

Many functions satisfy Definition~\ref{def:implication}, so these operators are typically classified by their construction methods, with numerous families documented in the literature (see \cite{Fernandez-Peralta2025} for a comprehensive survey). Beyond families, there exist many additional properties on fuzzy implication functions (see Table 1 in \cite{Fernandez-Peralta2025} for an overview). From this list, we focus on those most commonly used in the literature and relevant to our study.

\begin{itemize}[align=left]
    \item[\text{\NP}] \emph{Left neutrality}: \(I(1,y) = y\), \(y \in [0,1]\).
    \item[\text{\IP}] \emph{Identity principle}: \(I(x,x) = 1\), \(x \in [0,1]\).
    \item[\text{\OP}] \emph{Ordering property}: \(I(x,y)=1 \Leftrightarrow x \leq y\), \(x,y \in [0,1]\).
    \item[\text{\CB}] \emph{Consequent boundary}: \(I(x,y) \geq y\), \(x,y \in [0,1]\).
    \item[\text{\LF}] \emph{Lowest falsity}: \(I(x,y)=0 \Leftrightarrow x=1 \land y=0\).
    \item[\text{\LT}] \emph{Lowest truth}: \(I(x,y)=1 \Leftrightarrow x=0 \lor y=1\).
    \item[\text{\CPN}] \emph{Contrapositive symmetry} (w.r.t. \(N\)): \(I(x,y) = I(N(y),N(x))\), \(x,y \in [0,1]\).
    \item[\text{\LCPN}] \emph{Left contraposition} (w.r.t. \(N\)): \(I(N(x),y) = I(N(y),x)\), \(x,y \in [0,1]\).
    \item[\text{\RCPN}] \emph{Right contraposition} (w.r.t. \(N\)): \(I(x,N(y)) = I(y,N(x))\), \(x,y \in [0,1]\).
    \item[\text{\PIT}] \emph{\(T\)-power invariance} (w.r.t. continuous \(T\)): \(I(x,y) = I\left(x_T^{(r)}, y_T^{(r)}\right)\), for \(r > 0\), \(x,y \in (0,1)\) s.t. \(x_T^{(r)}, y_T^{(r)} \neq 0\), where the powers of the t-norm for any real value $r$ are defined as in \cite{Walker2002}.
\end{itemize}

We conclude this section by formalizing the $F$-chain-based construction method proposed in \cite{Mesiar2019}, accompanied by two illustrative examples.

\begin{definition}[\cite{Jin2019,Mesiar2019}]\label{def:fchain_construction}
    Let $F: [0,1]^n \to [0,1]$ be an aggregation function, and let $\bm{c} : [0,1] \to [0,1]^n$ be an increasing mapping such that $\bm{c}(0)=(0,\dots,0)$, $\bm{c}(1)=(1,\dots,1)$, and satisfying $F(\bm{c}(t))=t$ for all $t \in [0,1]$. Then $\bm{c}$ is called an $F$-chain. In this case, for each fuzzy implication function $I \in \mathcal{I}$, the function $I_{F,\bm{c}}:[0,1]^2 \to [0,1]$ given by
 \begin{equation}\label{eq:fchain}
    I_{F,\bm{c}}(x,y) = F(I(c_1(x),c_1(y)), \dots, I(c_n(x),c_n(y))),
 \end{equation}
 is a fuzzy implication function.
\end{definition}

\begin{example}\label{example:initial}
Two illustrative examples of the F-chain-based construction method are presented below.

\begin{itemize}\item[(i)]Let $F(x,y)=\max\{x,y\}$ and suppose $\bm{c}(t)=(t^2,t)$. Then, $\bm{c}$ is a $F$-chain. Now, if we consider the Kleene-Dienes implication $\IKD(x,y)=\max\{1-x,y\}$, with Eq. (\ref{eq:fchain}) we obtain the following fuzzy implication function
$$I(x,y)=\max\{\max\{1-x^2,y^2\},\max\{1-
x,y\}\} = \max\{1-x^2,y\}.$$
\item[(ii)] Let $F(x,y)=\min\{x,y\}$ and take $\bm{c}(t)=(c_1(t),c_2(t))$ where \\  $$c_1(t)=t, \quad c_2(t)=\left\{ \begin{array}{ll} 
t & \text{if }~ 0\leq t\leq 0.5, \\[0.5em]
2t-0.5 & \text{if }~ 0.5< t\leq 0.75, \\
[0.5em]
1 & \text{if }~ 0.75< t\leq 1.
\end{array}\right. $$
Then, $\bm{c}$ is a $F$-chain. If we consider the Łukasiewicz implication $\ILK(x,y)=\min\{1,1-x+y\}$  in Eq. (\ref{eq:fchain}) we obtain the fuzzy implication function 
$$I(x,y)=\min\{1, 1-x+y,1-c_2(x)+c_2(y)\},$$
where
$$1-c_2(x)+c_2(y)=
\left\{ \begin{array}{ll} 
x+y & \text{if }~ 0\leq x,y\leq 0.5, \\[0.5em] 
x+2y-0.5 & \text{if }~ 0\leq x\leq 0.5 \text{ and }  0.5< y\leq 0.75, \\[0.5em] 
x+1 & \text{if }~ 0\leq x\leq 0.5 \text{ and }  0.75< y\leq 1, \\[0.5em] 
2x+y-0.5 & \text{if }~ 0.5< x\leq 0.75 \text{ and }  0 \leq y\leq 0.5, \\[0.5em]
2(x+y)-1 & \text{if }~ 0.5< x,y\leq 0.75, \\[0.5em]
2x+0.5 & \text{if }~ 0.5< x\leq 0.75 \text{ and }  0.75 < y\leq 1, \\[0.5em]
1+y & \text{if }~ 0.75< x\leq 1 \text{ and }  0 \leq y\leq 0.5, \\[0.5em]
2y+0.5 & \text{if }~ 0.75< x\leq 1 \text{ and }  0.5 < y\leq 0.75, \\[0.5em]
2 & \text{if }~ 0.75< x,y\leq 1. \\[0.5em]
\end{array} \right.
$$
\end{itemize}
\end{example}

\section{Generalization of the $F$-chain-based construction method}\label{section:properties}

In this section, we provide a generalization of the $F$-chain-based construction method and we provide a thoroughly study of the preservation of several additional properties. Let us start by generalizing Definition \ref{def:fchain_construction} as follows.

 \begin{theorem}\label{th:def_method}
 Let $n \in \mathbb{N}$, $F \in \mathcal{A}_n$, $\bm{c}_i : [0,1] \to [0,1]^n$ with $i \in \{1,2\}$ be two increasing mappings such that $\bm{c}_i(0) =(0,\dots,0)$, $\bm{c}_i(1)=(1,\dots,1)$ and $I_1,\ldots,I_n\in \mathcal{I}$. Then, the function $(I_1,\ldots,I_n)_{F,\textbf{c}_1,\textbf{c}_2}:[0,1]^2\to[0,1]$ given by 
 \begin{equation}\label{eq:fchain_construction}
(I_1,\ldots,I_n)_{F,\textbf{c}_1,\textbf{c}_2}(x,y)=F(I_1(c_{1,1}(x),c_{2,1}(y)),\ldots,I_n(c_{1,n}(x),c_{2,n}(y))),
 \end{equation}
 is a fuzzy implication function. For short, we will denote \I sometimes as \Is and Eq. (\ref{eq:fchain_construction}) may be simplified using $\bigoplus$ to indicate concatenation as follows
 \begin{equation}\label{eq:fchain_construction:short}
 \Is (x,y) = F \left(\bigoplus_{i=1}^n I_i (c_{1,i}(x),c_{2,i}(y))\right).
 \end{equation}
 \end{theorem}

 \begin{proof}
     Let $F \in \mathcal{A}_n$, $\bm{c}_i : [0,1] \to [0,1]^n$ with $i \in \{1,2\}$ be two increasing mappings such that $\bm{c}_i(0) =(0,\dots,0)$, $\bm{c}_i(1)=(1,\dots,1)$ and $I_1,\dots,I_n \in \mathcal{I}$, we will prove that the function in Eq. (\ref{eq:fchain_construction}) fulfills the three conditions in Definition \ref{def:implication}:
    \begin{itemize}
    \item[\Ione] Let us consider $x_1, x_2 \in [0,1]$ such that $x_1 \leq x_2$ and $y \in [0,1]$, then since $F$ and $\bm{c}_1$ are increasing and each $I_i$ is decreasing with respect to the first variable we have 
    \begin{eqnarray*}
    \I(x_1,y) & = & F(I_1(c_{1,1}(x_1),c_{1,2}(y)),\ldots,I_n(c_{1,n}(x_1),c_{2,n}(y))) \\
    & \geq & F(I_1(c_{1,1}(x_2),c_{1,2}(y)),\ldots,I_n(c_{1,n}(x_2),c_{2,n}(y))) \\
    & = & \I(x_2,y).
    \end{eqnarray*}
    \item[\Itwo] Analogous to the proof of \Ione but taking into account that $\bm{c}_2$ is an increasing function and each $I_i$ is increasing with respect to the second variable.
    \item[\Ithree] The boundary conditions are directly derived from the boundary conditions of the corresponding aggregation function, $\bm{c}_1$, $\bm{c}_2$ and fuzzy implication functions:
    \begin{eqnarray*}
    \I(1,1) & = & F(I_1(c_{1,n}(1),c_{2,1}(1)),\ldots,I_n(c_{1,n}(1),c_{2,n}(1))) \\
    &=& F(I_1(1,1),\ldots,I_n(1,1)) = F(1,\dots,1)=1,
    \end{eqnarray*}
    \begin{eqnarray*}
    \I(0,0) & = & F(I_1(c_{1,1}(0),c_{2,1}(0)),\ldots,I_n(c_{1,n}(0),c_{2,n}(0))) \\
    &=& F(I_1(0,0),\ldots,I_n(0,0)) 
    = F(1,\dots,1)=1,
    \end{eqnarray*}
    \begin{eqnarray*}
    \I(1,0) & = & F(I_1(c_{1,1}(1),c_{2,1}(0)),\ldots,I_n(c_{1,n}(1),c_{2,n}(0))) 
    \\&=& F(I_1(1,0),\ldots,I_n(1,0)) = F(0,\dots,0)=0.
    \end{eqnarray*}
    \end{itemize}
 \end{proof}

 In Figure \ref{fig:diagram1} the reader can find a visual diagram that captures the generalized $F$-chain-based construction method. Notice that if in Theorem \ref{th:def_method} we consider $\bm{c}=\bm{c}_1 = \bm{c}_2$, $I=I_1=\dots=I_n$ and $F(\bm{c}(t))=t$ for all $t \in [0,1]$ we recover the construction method proposed in Definition \ref{def:fchain_construction}. Therefore, to ensure that from the construction method we obtain a fuzzy implication function, different increasing mappings can be used to transform each variable and a family of fuzzy implication functions can be used instead from a single one. In particular, $F(\bm{c}(t))=t$ in the definition of $F$-chain is not necessary for ensuring the conditions of a fuzzy implication function. Nonetheless, we will see that this condition is crucial for ensuring some properties that might be desirable, i.e., for ensuring the preservation of some additional properties either $\bm{c}_1$ or $\bm{c}_2$ have to be an $F$-chain.  On the other hand, we emphasize that the $F$-chain condition was introduced to generalize $\phi$-transformations of aggregation functions (where $\phi$ is an increasing automorphism on $[0,1]$), and in that framework it plays a crucial role in establishing key theoretical results.

 For the subsequent part of this section, we study the preservation of several additional properties for the proposed construction method.

 We start by considering the left neutral principle and the consequent boundary. In this case, when all the members of the family of fuzzy implication functions satisfy these properties and $\bm{c}_2$ is an $F$-chain, the construction preserves them. Then, we provide two examples to illustrate how the generalized $F$-chain-based construction method may generate new fuzzy implication functions that satisfy these properties.

 \begin{figure}[t]
\centering
\includegraphics[scale=0.2]{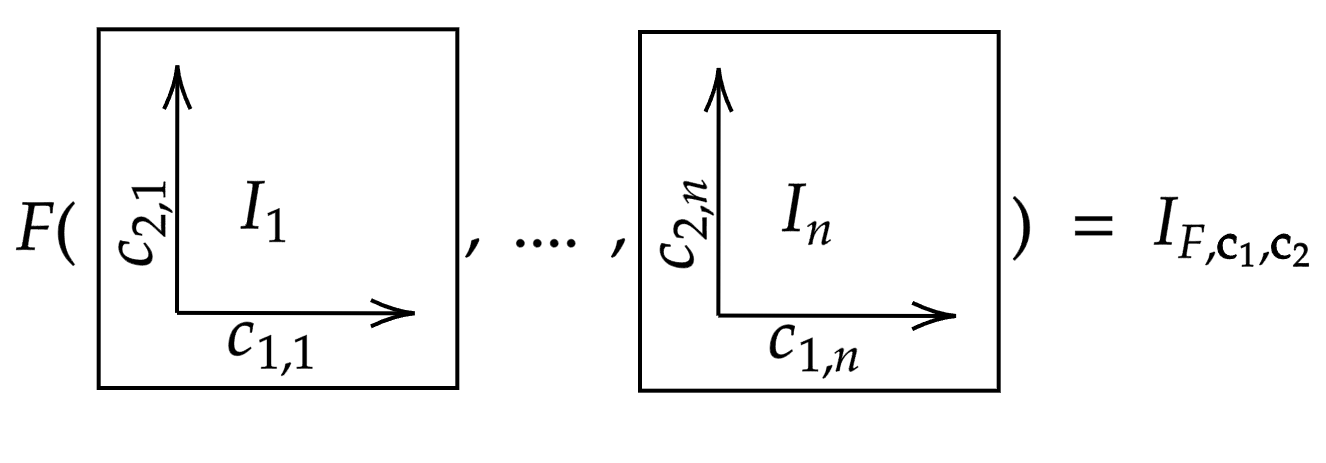}
\caption{Graphic diagram of the generalized $F$-chain-based construction method.}\label{fig:diagram1}
\end{figure}


\begin{proposition}\label{prop:consequent_boundary}
Let $n \in \mathbb{N}$, $F \in \mathcal{A}_n$, $\mathbf{c_2}$ an $F$-chain and $I_1,\ldots,I_n\in \mathcal{I}$. If $F(\bm{c}_2(t))=t$ for all $t\in [0,1]$ then the following statements hold:
\begin{enumerate}[(i)]
\item If $I_1,\ldots,I_n\in \mathcal{I}$ satisfy \NP, then \I satisfies \NP.
\item If $I_1,\ldots,I_n\in \mathcal{I}$ satisfy \CB, then \I satisfies \CB.
\end{enumerate}
\end{proposition}

\begin{proof}~
\begin{enumerate}[(i)]
 \item Let $y \in [0,1]$, then
 \begin{eqnarray*}
 \I(1,y) &=& F(I_1(c_{1,1}(1),c_{2,1}(y)),\ldots,I_n(c_{1,n}(1),c_{2,n}(y))) \\
 &=& F(I_1(1,c_{2,1}(y)),\ldots,I_n(1,c_{2,n}(y))) \\
 & = & F(c_{2,1}(y),\ldots,c_{2,n}(y)) = F(\bm{c}_2(y))= y.
 \end{eqnarray*}
\item Let $x,y \in [0,1]$, since each $I_i$ satisfies \CB we have $I_i(c_{1,i}(x),c_{2,i}(y)) \geq c_{2,i}(y)$ and since $F$ is increasing we obtain
     \begin{eqnarray*}
    \I(x,y) &=& F(I_1(c_{1,1}(x),c_{2,1}(y)), \dots, I_n(c_{1,n}(x),c_{2,n}(y))) \\
    &\geq & F(c_{2,1}(y), \dots, c_{2,n}(y)) \\
    &=& F(\bm{c}_2(y)) = y.
    \end{eqnarray*}
\end{enumerate}
\end{proof}

\begin{example}\label{example:CB-NP}
Let us consider the Łukasiewicz and Kleene-Dienes implications (see Example \ref{example:initial}), we know that both satisfy \NP and \CB \cite{Baczynski2008}. Now, we consider $F(x,y)=xy$, $\bm{c}_1=(2t-t^2,t)$ and   $\bm{c}_2=(\sqrt{t},\sqrt{t})$.  In this case
$$
F(c_{2,1}(t),c_{2,2}(t))=\sqrt{t} \cdot \sqrt{t} =t,
$$
and $\bm{c}_2$ is an $F$-chain. Then, in this case the fuzzy implication function obtained via Eq. (\ref{eq:fchain_construction}) satisfies both \NP and \CB. Indeed, we have
\begin{eqnarray*}
(I_1,I_2)_{F,\bm{c}_1,\bm{c}_2}(x,y) &=& F(\ILK(2x-x^2,\sqrt{y}),\IKD(x,\sqrt{y})) \\
&=&\min\{1,1-2x+x^2+\sqrt{y}\}\max\{1-x,\sqrt{y}\},
\end{eqnarray*}
with,
$(I_1,I_2)_{F,\bm{c}_1,\bm{c}_2}(1,y)=\min\{1,\sqrt{y}\}\max\{0,\sqrt{y}\}=y$ and $(I_1,I_2)_{F,\bm{c}_1,\bm{c}_2}(x,y)= \min\{1,1-2x+x^2+\sqrt{y}\}\max\{1-x,\sqrt{y}\}\geq \sqrt{y}\sqrt{y}=y$ for all $x,y \in [0,1]$.
\end{example}

Next, we investigate the structure of the natural negation in relation to the natural negations of a collection of fuzzy implication functions. Interestingly, when all the original fuzzy implication functions share the same strong natural negation, $\bm{c}_1$ forms an $F$-chain, and $F$ is a self $N$-dual aggregation function, the construction method preserves the same natural negation.

 \begin{proposition}\label{prop:natural_negation}
 Let $n \in \mathbb{N}$, $F \in \mathcal{A}_n$, $\bm{c}_i : [0,1] \to [0,1]^n$ with $i \in \{1,2\}$ be two increasing mappings such that $\bm{c}_i(0) =(0,\dots,0)$, $\bm{c}_i(1)=(1,\dots,1)$ and $I_1,\ldots,I_n\in \mathcal{I}$. Then, the natural negation of \I is given by
\begin{equation}\label{eq:natural_negation}
N_{\I}(x) = F(N_{I_1}(c_{1,1}(x)), \dots, N_{I_n}(c_{1,n}(x))).
\end{equation}
Moreover, if $N_{I_i} = N$ for all $i \in \{1,\dots,n\}$, where $N$ is a strong negation, $F(\bm{c}_1(t))=t$ for all $t \in [0,1]$ and $F$ is a self $N$-dual aggregation function, then $N_{\I}(x) = N$.
\end{proposition}

\begin{proof}
First, let us prove Eq. (\ref{eq:natural_negation})
\begin{eqnarray*}
N_{\I}(x) &=& \I(x,0) = F(I_1(c_{1,1}(x),c_{2,1}(0)),\dots, I_n(c_{1,n}(x),c_{2,n}(0))) \\
& = & F(I_1(c_{1,1}(x),0),\dots, I_n(c_{1,n}(x),0)) \\
& = & F(N_{I_1}(c_{1,1}(x)),\dots, N_{I_n}(c_{1,n}(x))).
\end{eqnarray*}
On the other hand, let us assume that $N$ is a strong negation such that $N_{I_i} = N$ for all $i \in \{1,\dots,n\}$ and $F$ a self $N$-dual aggregation function, then
\begin{eqnarray*}
F(N_{I_1}(c_{1,1}(x)),\dots, N_{I_n}(c_{1,n}(x))) & = & F(N(c_{1,1}(x)),\dots, N(c_{1,n}(x))) \\ 
& = & N(F(c_{1,1}(x), \dots, c_{1,n}(x))) \\
&=& N(F(\bm{c}_1(x))) = N(x).
\end{eqnarray*}
\end{proof}
\begin{example}\label{example:negation}
Let us consider the classical negation $N_c(x)=1-x$, $I_1(x,y)=\min(1,1-x+y)$ the \L ukasiewicz implication, $I_2(x,y)=1-x+xy$ the Reichenbach implication and $I_3(x,y)=\max(1-x,y)$ the Kleene-Dienes implication. These three fuzzy implication functions have as their natural negation the classical negation $N_c$. Now, let us consider as $F$ the following self $N_c$-dual aggregation function $F(x_1,x_2,x_3)=\frac{\max(x_1,x_2,x_3)+\min(x_1,x_2,x_3)}{2}$ (see \cite{GARCIALAPRESTA2008}), the $F$-chain $\mathbf{c_1}(t)=(t^2,2t-t^2,t)$ and the increasing mapping $\mathbf{c_2}(t)=(t,t^2,t^3)$ (which is not an $F$-chain). In these conditions, and according to  Proposition \ref{eq:natural_negation}, we have \begin{eqnarray*}
F(N_{I_1}(c_{1,1}(x)),N_{I_2}(c_{1,1}(x)), N_{I_3}(c_{1,3}(x))) & = & N_c(x).
\end{eqnarray*}
\end{example}

Subsequently, we study the identity and ordering principles. In this case, we also find sufficient conditions involving $\bm{c}_1$ and $\bm{c}_2$ for the properties to be preserved.
\begin{proposition} Let $n \in \mathbb{N}$, $F \in \mathcal{A}_n$, $\bm{c}_i : [0,1] \to [0,1]^n$ with $i \in \{1,2\}$ be two increasing mappings such that $\bm{c}_i(0) =(0,\dots,0)$, $\bm{c}_i(1)=(1,\dots,1)$ and $I_1,\ldots,I_n\in \mathcal{I}$.  If $c_{1,j}(x) \leq c_{2,j}(x)$ for all $x \in [0,1]$ and $j \in \{1,\dots,n\}$ then the following statement holds: if $I_1,\ldots,I_n$ satisfy \IP then \I also satisfies \IP.
 \end{proposition}
 \begin{proof}
 \item Let $x \in [0,1]$, then
 \begin{eqnarray*}
 \I(x,x) &=& F(I_1(c_{1,1}(x),c_{2,1}(x)),\ldots,I_n(c_{1,n}(x),c_{2,n}(x))) \\
 & \geq & F(I_1(c_{2,1}(x),c_{2,1}(x)),\ldots,I_n(c_{2,n}(x),c_{2,n}(x))) \\
 &=& F(1,\ldots,1)= 1.
 \end{eqnarray*}
 \end{proof}
\begin{proposition}\label{prop:ordering_property}
Let $n \in \mathbb{N}$, $F \in \mathcal{A}_n$, $\bm{c}_i : [0,1] \to [0,1]^n$ with $i \in \{1,2\}$ be two increasing mappings such that $\bm{c}_i(0) =(0,\dots,0)$, $\bm{c}_i(1)=(1,\dots,1)$ and $I_1,\ldots,I_n\in \mathcal{I}$. If $c_{1,j}(x) \leq c_{2,j}(x)$ for all $x \in [0,1]$ and $j \in \{1,\dots,n\}$, $F$ has no unit multipliers and $I_1,\ldots,I_n\in \mathcal{I}$ satisfy \OP, then \I satisfies \OP.
\end{proposition}
\begin{proof}
\begin{eqnarray*}
1 = \I(x,y) & \Leftrightarrow & 1 = F(I_1(c_{1,1}(x),c_{2,1}(y)), \dots, I_n(c_{1,n}(x),c_{2,n}(y)))\\
& \Leftrightarrow & 
1 = I_j(c_{1,j}(x),c_{2,j}(y)) \text{ for all } j \in \{1,\dots,n\} \\
& \Leftrightarrow & c_{1,j}(x) \leq c_{2,j}(y) \text{ for all } j \in \{1,\dots,n\} \\
& \Leftrightarrow & x \leq y,
\end{eqnarray*}
since $c_{1,j}$, $c_{2,j}$ are increasing and $c_{1,j}(x) \leq c_{2,j}(x)$.
\end{proof}


We proceed to study in which cases the generalization of the $F$-chain-based construction method preserves any of the contrapositions with respect to any fuzzy negation. In this case, it is sufficient that the two increasing functions, $\bm{c}_1$ and $\bm{c}_2$ are the same one and they have to commute with the considered negation.

\begin{proposition}\label{prop:contrapositivisations}
    Let $n \in \mathbb{N}$, $F \in \mathcal{A}_n$, $\bm{c}_i : [0,1] \to [0,1]^n$ with $i \in \{1,2\}$ be two increasing mappings such that $\bm{c}_i(0) =(0,\dots,0)$, $\bm{c}_i(1)=(1,\dots,1)$ and $I_1,\ldots,I_n\in \mathcal{I}$. If $\bm{c}_1=\bm{c}_2=\bm{c}$ and $N$ is fuzzy negation such that $c_j(N(t)) = N(c_j(t))$ for all $j \in \{1,\dots,n\}$, $t \in [0,1]$ then the following statements hold:
    \begin{enumerate}[(i)]
    \item If $I_1,\ldots,I_n\in \mathcal{I}$ satisfy \CPN, then \I satisfies \CPN.
    \item If $I_1,\ldots,I_n\in \mathcal{I}$ satisfy \LCPN, then \I satisfies \LCPN.
    \item If $I_1,\ldots,I_n\in \mathcal{I}$ satisfy \RCPN, then \I satisfies \RCPN.
    \end{enumerate}
\end{proposition}

\begin{proof} First, we prove Point (i).
    \begin{eqnarray*}
    \I(N(y),N(x)) &=& F(I_1(c_{1}(N(y)),c_{1}(N(x))), \dots, I_n(c_{n}(N(y)),c_{n}(N(x)))) \\
    &=& F(I_1(N(c_{1}(y)),N(c_{1}(x))), \dots, I_n(N(c_{1}(y)),N(c_n(x)))) \\
    &=& F(I_1(c_1(x),c_1(y)), \dots, I_n(c_n(x),c_n(y))) \\
    &=& \I(x,y) \\
    \end{eqnarray*}
Now, we verify Point (ii).
    \begin{eqnarray*}
    \I(N(x),y) &=& F(I_1(c_1(N(x)),c_1(y)), \dots, I_n(c_n(N(x)),c_n(y))) \\
    &=& F(I_1(N(c_1(x)),c_1(y)), \dots, I_n(N(c_n(x)),c_n(y))) \\
    &=& F(I_1(N(c_1(y)),c_1(x)), \dots, I_n(N(c_n(y)),c_n(x))) \\
    &=& F(I_1(c_1(N(y)),c_1(x)), \dots, I_n(c_n(N(y)),c_n(x))) \\
    &=& \I(N(y),x).
    \end{eqnarray*}
Finally, Point (iii) is analogous to the proof of Point (ii).
\end{proof}

\begin{example}\label{ex:cpn} 
Let $N_c(x)=1-x$ be the classical fuzzy negation, we consider
$$
c_j(x) = \sin^2\left(\frac{\pi \cdot x}{2}\right), \quad \text{for all } x \in [0,1] \text{ and } j \in \{1,\dots,n\}.
$$
It is clear that all $c_j$ are increasing, $c_j(0)=0$, $c_j(1)=1$ and
$$
c_j(x) + c_j(1-x)=\sin^2\left(\frac{\pi \cdot x}{2}\right) + \sin^2\left(\frac{\pi \cdot (1-x)}{2}\right)=1.
$$
Then, for any family of fuzzy implication functions $I_1,\ldots,I_n\in \mathcal{I}$ such that satisfy \CPN, \LCPN or \RCPN with respect to $N_c$, since $N_c$ is a strong negation by \cite[Proposition 1.4.3]{Baczynski2008} each $I_i$ satisfies the three contrapositions with respect to $N_c$ and for any $n$-ary aggregation function $F$ the fuzzy implication function \I will satisfy \CPN, \LCPN or \RCPN, respectively. For instance, let us consider $F$ as the weighted arithmetic mean with respect to $\bm{w}$ and $I_i$ be an $(S,N)$-implication generated by a t-conorm $S_i$ and $N=N_c$, then 
$$
I(x,y)= \sum_{i=1}^n w_iS_i\left(\sin^2\left(\frac{\pi \cdot (1-x)}{2}\right),\sin^2\left(\frac{\pi \cdot y}{2}\right)\right),
$$
is a fuzzy implication function that satisfies the three contrapositions with respect to $N_c$.
In Figure \ref{fig:cpn} the reader can find a graphical example in which each $S_i$ is a Yager t-conorm  $ S_{\lambda_i}^Y$ with $\lambda_i = \frac{1}{i}$ (see \cite[Section A.8]{Klement2000}) and $w_i = \left(\frac{1}{3},\frac{1}{3},\frac{1}{3}\right)$.

\begin{figure}[h]
\centering
\includegraphics[scale=0.35]{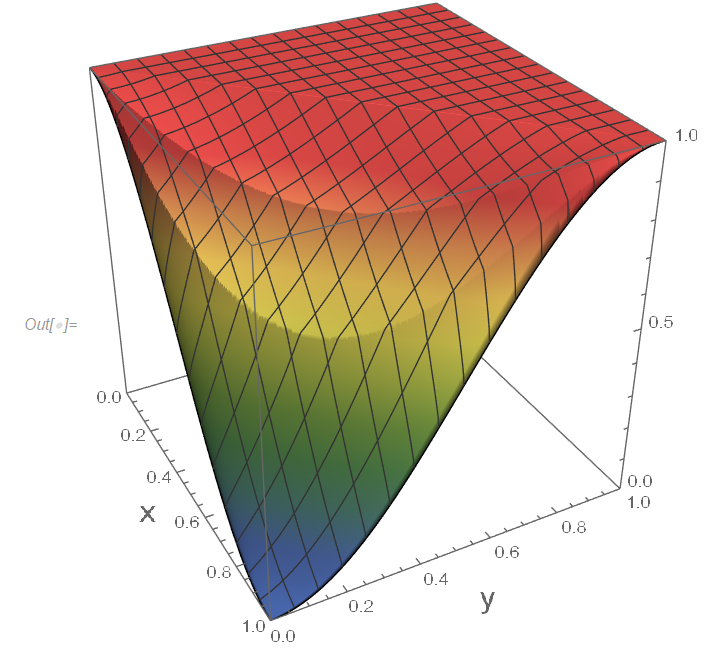}
\caption{Contruction of fuzzy implication function in Example \ref{ex:cpn}.}\label{fig:cpn}
\end{figure}
\end{example}

We now study the lowest truth and falsity properties. In this case, we need to use an aggregation function which has no zero or unit multipliers, and functions $\bm{c}_1$ and $\bm{c}_2$ which only value 0 or 1 on the boundaries.

\begin{proposition}\label{prop:lowest_truth_falsity}
Let $n \in \mathbb{N}$, $F \in \mathcal{A}_n$, $\bm{c}_i : [0,1] \to [0,1]^n$ with $i \in \{1,2\}$ be two increasing mappings such that $\bm{c}_i(0) =(0,\dots,0)$, $\bm{c}_i(1)=(1,\dots,1)$ and $I_1,\ldots,I_n\in \mathcal{I}$. If $c_{i,j}(t) \in (0,1)$ for all $t \in (0,1)$, $i \in \{1,2\}$ and 
$j \in \{1,\dots,n\}$, $F$ has no unit multipliers (resp. zero multipliers) and $I_1,\ldots,I_n\in \mathcal{I}$ satisfy \LT (resp. \LF), then \I satisfies \LT (resp. \LF).
\end{proposition}

\begin{proof}
\begin{eqnarray*}
1 = \I(x,y) & \Leftrightarrow & 1 = F(I_1(c_{1,1}(x),c_{2,1}(y)), \dots, I_n(c_{1,n}(x),c_{2,n}(y)))\\
& \Leftrightarrow & 
1 = I_j(c_{1,j}(x),c_{2,j}(y)) \text{ for all } j \in \{1,\dots,n\} \\
& \Leftrightarrow & c_{1,j}(x)=0 \text{ or } c_{2,j}(y)=1 \text{ for all } i \in \{1,\dots,n\} \\
& \Leftrightarrow & x=0 \text{ or } y=1.
\end{eqnarray*}
\end{proof}

Finally, we consider the invariance with respect to powers of continuous t-norms, a property which has been marked as valuable for approximate reasoning \cite{Massanet2017}.

\begin{proposition} Let $n \in \mathbb{N}$, $F \in \mathcal{A}_n$, $\bm{c}_i : [0,1] \to [0,1]^n$ with $i \in \{1,2\}$ be two increasing mappings such that $\bm{c}_i(0) =(0,\dots,0)$, $\bm{c}_i(1)=(1,\dots,1)$ and $I_1,\ldots,I_n\in \mathcal{I}$ and $T$ a continuous t-norm. If $c_{i,j}(x_{T}^{(r)}) = c_{i,j}(x)_T^{(r)}$ for all $i \in \{1,2\}$, $j \in \{1,\dots,n\}$, $r>0$ and $x \in (0,1)$ such that $x_{T}^{(n)} \not = 0$ and $I_1, \dots, I_n$ satisfy \PIT, then \I satisfies \PIT.
\end{proposition}
\begin{proof}
\begin{eqnarray*}
\I\left(x_T^{(r)}, y_T^{(r)}\right) & = & F \left(\bigoplus_{i=1}^n I_i (c_{1,i}(x_T^{(r)}),c_{2,i}(y_T^{(r)}))\right) \\
& = & F \left(\bigoplus_{i=1}^n I_i (c_{1,i}(x)_T^{(r)},c_{2,i}(y)_T^{(r)})\right)
\\
& = & F \left(\bigoplus_{i=1}^n I_i (c_{1,i}(x),c_{2,i}(y))\right) \\
& = & F(x,y).
\end{eqnarray*}
\end{proof}

\begin{example} Let us consider $T$ and $F$ as the product t-norms and $c_{i,j}(x)=x^{i+j}$ for all $x \in [0,1]$. Then, for any collection of fuzzy implication functions $I_1,\dots,I_n \in \mathcal{I}$ that satisfy \PIT we have that the resulting implication $\I$ also satisfies \PIT. Indeed, in this case
$$c_{i,j}(x_T^{(r)}) = c_{i,j}(x^r) = x^{r \cdot (i+j)} = (x^{i+j})^r = c_{i,j}(x)^r = c_{i,j}(x)_T^{(r)}.$$
\end{example}

To conclude this section, we emphasize that many other relevant properties, such as the exchange principle, the law of importation, $T$-conditionality, or the iterative Boolean law, have not been addressed here. These properties are considerably more challenging to analyze for analogous construction methods, including aggregation, horizontal/vertical constructions, or ordinal sums. Indeed, such properties often require dedicated studies \cite{Cao2022,Su2016,Yi2024} and/or involve complex conditions which are not easy to verify \cite{Reiser2013,Su2015b}. For these reasons, we leave their investigation as future work.

\section{Generalized $F$-chain-based construction method as generalization of other construction methods}\label{section:construction_methods}

In Section \ref{section:properties} we have introduced a new construction method for fuzzy implication function based on the rescaling of each variable and the posterior aggregation. Although this method is already a new contribution to the literature, the main interest of the method lies in the fact that it can encompass different construction methods in a single compact expression. To demonstrate this fact, in this section we consider a wide variety of construction methods and show that by selecting an appropriate aggregation function and input transformation defined by $\bm{c}_1$, $\bm{c}_2$, the method can be expressed as the generalized $F$-chain method.

In this paper, we focus on construction methods that involve at least two different fuzzy implication functions since it is the scenario in which the generalization of the $F$-chain-based construction is more meaningful. Also, due to the wide variety of construction methods considered we do not recall each definition here. To see an overview of several construction methods and consult their expressions, we refer the reader to \cite[Table 14]{Fernandez-Peralta2025}.

\subsection{Aggregation of fuzzy implication functions}

First of all, we consider the aggregation of fuzzy implication functions which was introduced in \cite{Reiser2013} which encompasses the popular classical methods of minimum, maximum and convex combination. In this framework, it becomes evident that the generalized $F$-chain-based construction method with the identity as input transformation is equivalent to these methods, selecting the appropriate aggregation function in each case.

\begin{proposition}
Let $n \in \mathbb{N}$ and let $\{I_i\}_{i=1}^n$ be a collection of fuzzy implication functions.

\begin{itemize}
\item If $F(x_1,\dots,x_n) = \max \{x_1,\dots,x_n\}$ for all $x_i \in [0,1]$, $i \in \{1,\dots,n\}$ and $c_{1,i}(x)=c_{2,i}(x)=x$ for all $x \in [0,1]$, $i \in \{1,\dots,n\}$ then
$$\Is(x,y) = \max\{I_1(x,y),\dots,I_n(x,y)\} = (I_1 \vee\ldots \vee I_n)(x,y), \quad \text{for all } x,y \in [0,1].$$
\item If $F(x_1,\dots,x_n) = \min \{x_1,\dots,x_n\}$ for all $x_i \in [0,1]$, $i \in \{1,\dots,n\}$ $c_{1,i}(x)=c_{2,i}(x)=x$ for all $x \in [0,1]$, $i \in \{1,\dots,n\}$ then
$$\Is(x,y) = \min\{I_1(x,y),\dots,I_n(x,y)\} = (I_1 \wedge\ldots \wedge I_n)(x,y), \quad \text{for all } x,y \in [0,1].$$
\item Let $\{\lambda_i\}_{i=1}^n$ be a sequence of values in $[0,1]$ such that $\displaystyle \sum_{i=1}^n \lambda_i =1$. If $F(x_1,\dots,x_n) = \sum_{i=1}^n \lambda_ix_i$ for all $x_i \in [0,1]$ $c_{1,i}(x)=c_{2,i}(x)=x$ for all $x \in [0,1]$, $i \in \{1,\dots,n\}$ then
$$\Is(x,y) = \sum_{i=1}^n \lambda_i I_i(x,y) = I_{I_1,\dots,I_n}^{\lambda_1,\dots,\lambda_n}(x,y), \quad \text{for all } x,y \in [0,1].$$
\item If $c_{1,i}(x)=c_{2,i}(x)=x$ for all $x \in [0,1]$, $i \in \{1,\dots,n\}$ then
$$\Is(x,y) = F\{I_1(x,y),\dots,I_n(x,y)\} = (I_1 ,\ldots ,I_n)_F(x,y), \quad \text{for all } x,y \in [0,1].$$
\end{itemize}
\end{proposition}

Interestingly, thanks to this assignment we can use the results in Section \ref{section:properties} to directly disclose sufficient conditions for the considered additional properties to the construction methods based on the aggregation of fuzzy implication functions. Indeed, in Table \ref{table:agg_properties} there is a summary of all the properties considered in Section \ref{section:properties}. These results align with existing studies of these methods, and even extend them for those properties that have not been previously considered.


\begin{table}[ht]
\centering
\resizebox{\textwidth}{!}{%
\begin{tabular}{|c|c|c|c|c|c|}
\hhline{|~|-|-|-|-|-|}
\multicolumn{1}{c|}{} &
\textbf{Generalized $F$-chain} & \textbf{Max} & \textbf{Min} & \textbf{Convex} & \textbf{Aggregation} \\ \hline

\multicolumn{1}{|c|}{$F$} &
$F$ & $\max$ & $\min$ & $\displaystyle \sum_{i=1}^n \lambda_i x_i$ & $F$ \\ \hline

\multicolumn{1}{|c|}{$\bm{c}_1$, $\bm{c}_2$} &
$\bm{c}_1$, $\bm{c}_2$ & $c_{i,j}(x)=x$ & $c_{i,j}(x)=x$ & $c_{i,j}(x)=x$ & $c_{i,j}(x)=x$ \\ \hline

\multicolumn{1}{|c|}{\NP} &
$F(\bm{c}_2(t))=t$ & \cmark & \cmark & \cmark & $F(x,\dots,x)=x$ \\ \hline

\multicolumn{1}{|c|}{\CB} &
$F(\bm{c}_2(t))=t$ & \cmark & \cmark & \cmark & $F(x,\dots,x)=x$ \\ \hline

\multicolumn{1}{|c|}{\IP} &
$c_{1,j}(x) \leq c_{2,j}(x)$ & \cmark & \cmark & \cmark & \cmark \\ \hline

\multicolumn{1}{|c|}{\CPN} &
$\bm{c}_1 = \bm{c}_2 = \bm{c}$, $c_j(N(x)) = N(c_j(x))$ & \cmark & \cmark & \cmark & \cmark \\ \hline

\multicolumn{1}{|c|}{\LCPN} &
$\bm{c}_1 = \bm{c}_2 = \bm{c}$, $c_j(N(x)) = N(c_j(x))$ & \cmark & \cmark & \cmark & \cmark \\ \hline

\multicolumn{1}{|c|}{\RCPN} &
$\bm{c}_1 = \bm{c}_2 = \bm{c}$, $c_j(N(x)) = N(c_j(x))$ & \cmark & \cmark & \cmark & \cmark \\ \hline

\multicolumn{1}{|c|}{\OP} &
\shortstack{$c_{1,j}(x) \leq c_{2,j}(x)$ \\ $F$ no unit mult.} & ? & \cmark & \cmark & $F$ no unit mult \\ \hline

\multicolumn{1}{|c|}{\LT} &
\shortstack{$c_{i,j}(x) \in (0,1)$, $x \in (0,1)$ \\ $F$ no unit mult.} & ? & \cmark & ? & $F$ no unit mult. \\ \hline

\multicolumn{1}{|c|}{\LF} &
\shortstack{$c_{i,j}(x) \in (0,1)$, $x \in (0,1)$ \\ $F$ no zero mult.} & \cmark & ? & ? & $F$ no zero mult. \\ \hline

\multicolumn{1}{|c|}{\PIT} &
$c_{i,j}(x_T^{(n)}) = c_{i,j}(x)_T^{(n)}$ & \cmark & \cmark & \cmark & \cmark \\ \hline
\end{tabular}%
}
\caption{In each cell, we can find the sufficient conditions of the additional property of the corresponding row and the construction method of the corresponding column. If a check mark \cmark is used, then the property is always satisfied. If an interrogation symbol (?) is used then the sufficient conditions do not apply to this case.}
\label{table:agg_properties}
\end{table}

\subsection{Contrapositivisations}

In this section, we consider contrapositivisation methods which originated under the motivation of modifying a fuzzy implication function which may not satisfy \CPN in order to satisfy this property. Although these methods only use one fuzzy implication function, they apply transformations of this fuzzy implication function using a fuzzy negation. These transformations are, in turn,  also fuzzy implication functions. Thus, the most well-known contrapositivisation methods can also be rewritten using the generalized $F$-chain-based construction method.


\begin{proposition} Let $N$ be a fuzzy negation. Then,
\begin{enumerate}[(i)]
\item If $n=2$, $F = \min$, $I_1=I$, $I_2=I_N$ and $c_{1,i}(x)=c_{2,i}(x)=x$ for all $x \in [0,1]$, $i \in \{1,\dots,n\}$ then
$$
\Is(x,y) = \min \{I(x,y),I(N(y),N(x))\} = I_N^U(x,y).
$$
\item If $n=2$, $F = \max$, $I_1=I$, $I_2=I_N$ and $c_{1,i}(x)=c_{2,i}(x)=x$ for all $x \in [0,1]$, $i \in \{1,\dots,n\}$ then
$$
\Is(x,y) = \max \{I(x,y),I(N(y),N(x))\} = I_N^L(x,y).
$$
\item If $n=2$, $F = \min$, $I_1(x,y)=\max\{I(x,y),N(x)\}$, $I_2(x,y)=\max\{I(N(y),N(x)),y\}$ for all $x,y \in [0,1]$ and $c_{1,i}(x)=c_{2,i}(x)=x$ for all $x \in [0,1]$, $i \in \{1,\dots,n\}$ then
$$
\Is(x,y) = \min \{ \max\{I(x,y),N(x)\},\max\{I(N(y),N(x)),y\}\} = I_N^M(x,y).
$$
\end{enumerate}
\end{proposition}

\begin{proof}
Since the $N$-reciprocation of a fuzzy implication function given by $I_N(x,y)=I(N(y),N(x))$ is a fuzzy implication function (see \cite[Definition 1.6.1]{Baczynski2008}), Points (i) and (ii) follow directly from Definition \ref{def:fchain_construction}. For Point (iii) we just need to prove that $I_1$ and $I_2$ are fuzzy implication functions. Indeed, let us prove this fact only for $I_1$ since for $I_2$ is analogous:
\begin{itemize}
\item $I_1(0,0) = \max\{I(0,0),N(0)\} = \max\{1,1\} =1$.
\item $I_1(1,1) = \max\{I(1,1),N(1)\} = \max\{1,0\} =1$.
\item $I_1(1,0) = \max\{I(1,0),N(1)\} = \max\{0,0\} =0$.
\item Let $x_1,x_2,y \in [0,1]$ with $x_1 \leq x_2$. Then, $I(x_1,y) \geq I(x_2,y)$ and $N(x_1) \geq N(x_2)$, so we obtain
$$
I_1(x_1,y) = \max\{I(x_1,y),N(x_1)\} \geq \max\{I(x_2,y),N(x_2)\} = I_1(x_2,y).
$$
\item Let $x,y_1,y_2 \in [0,1]$ with $y_1 \leq y_2$. Then $I(x,y_1) \leq I(x,y_2)$ so we obtain
$$
I_1(x,y_1) = \max \{ I(x,y_1), N(x) \} \leq \max \{ I(x,y_2), N(x) \} = I_1(x,y_2).
$$
\end{itemize}
\end{proof}

In this case, the preservation of additional properties directly depends on $I_N$, $I_1$ and $I_2$. We provide here such study only for the $N$-reciprocation.

\begin{lemma}\label{lem:nreciprocication}
Let $N$ be a fuzzy negation and $I$ a fuzzy implication function. If we consider the $N$-reciprocation $I_N(x,y)=I(N(y),N(x))$, then the following statements hold:
\begin{enumerate}[(i)]
\item If $N_I \circ N = \text{id}_{[0,1]}$ then $I_N$ satisfies \NP.
\item $I_N$ does not generally satisfy \CB.
\item If $I$ satisfies \OP (resp. \IP), then $I_N$ satisfies also \IP (resp. \OP).
\item If $I$ satisfies \CPN (resp. \LCPN or \RCPN), then $I_N$ satisfies \CPN (resp. \LCPN or \RCPN).
\item If $I$ satisfies \LF (resp. \LT) and $N(x) \in (0,1)$ for all $x \in (0,1)$, then $I_N$ satisfies \LF (resp. \LT).
\item Let $T$ be a continuous t-norm. If $I$ satisfies \PIT and $N(x_T^{(r)}) = N(x)_T^{(r)}$ for $r>0$, $x\in (0,1)$ such that $x_T^{(r)} \not = 0$, then $I_N$ satisfies \PIT.
\end{enumerate}
\end{lemma}

\begin{proof}
\begin{enumerate}[(i)]
\item \cite[Proposition 1.6.4]{Baczynski2008}.
\item Let $N(x)=1-x^2$ and $I=\IRC$ the Reichenbach implication, then
$$I_N(x,y) = \IRC(1-x^2,1-y^2) = 1-x^2+x^2y^2,$$
and $I_N(0.8,0.8) = 0.7696 < 0.8.$
\item \cite[Proposition 1.6.6]{Baczynski2008}.
\item 
$$ I_N(x,y) = I(N(y),N(x)) = I(N(N(x)),N(N(y))) = I_N(N(y),N(x)),$$
$$ I_N(x,N(y)) = I(N(N(y)),N(x)) = I(N(N(x)),N(y)) = I_N(y,N(x)),$$
$$ I_N(N(x),y) = I(N(y),N(N(x))) = I(N(x),N(N(y)) = I_N(N(y),x).$$
\item 

$$ I_N(x,y) = 0 \Leftrightarrow I(N(y),N(x)) = 0 \Leftrightarrow N(y)=1 \wedge N(x) = 0 \Leftrightarrow y=0 \wedge x=1,$$
$$ I_N(x,y) = 1 \Leftrightarrow I(N(y),N(x)) = 1 \Leftrightarrow N(y) = 0 \vee N(x)=1 \Leftrightarrow y=1 \vee x=0.$$
\item $$ I_N(x_T^{(r)},y_T^{(r)}) = I(N(y_T^{(r)}),N(x_T^{(r)})) = I(N(y)_T^{(r)},N(x)_T^{(r)})=I(N(y),N(x)) = I_N(x,y).$$ 
\end{enumerate}
\end{proof}

Consequently, thanks to the results in Section \ref{section:properties} and Lemma \ref{lem:nreciprocication} we can directly retrieve sufficient conditions for the preservation of the additional properties of the low and upper contrapositivisations. We summarize the results in Table \ref{table:contrapositivisations}.

\begin{table}[h]
\centering
\resizebox{\textwidth}{!}{%
\setlength{\tabcolsep}{4pt}
\renewcommand{\arraystretch}{1.5}
\begin{tabular}{|c|c|c|c|c|c|c|c|c|c|c|}
\hhline{|~|-|-|-|-|-|-|-|-|-|-|}
\multicolumn{1}{c|}{} & \NP & \CB & \IP & \CPN & \LCPN & \RCPN & \OP & \LT & \LF & \PIT \\ \hline
\multicolumn{1}{|c|}{$I_N^U$} &
$N_I \circ N = \text{id}_{[0,1]}$ & ? & \cmark & \cmark & \cmark & \cmark & \cmark & \cmark & ? &
$N(x_T^{(r)}) = N(x)_T^{(r)}$ \\ \hline
\multicolumn{1}{|c|}{$I_N^L$} &
$N_I \circ N = \text{id}_{[0,1]}$ & ? & \cmark & \cmark & \cmark & \cmark & ? & ? & \cmark &
$N(x_T^{(r)}) = N(x)_T^{(r)}$ \\ \hline
\end{tabular}%
}
\caption{In each cell, we can find the sufficient conditions of the additional property of the corresponding column and the construction method of the corresponding row. If a check mark \cmark is used, then the property is always satisfied. If an interrogation symbol (?) is used then the sufficient conditions do not apply to this case.}
\label{table:contrapositivisations}
\end{table}

\subsection{Threshold methods}\label{subsection:threshold_methods}

Now, we consider the vertical and horizontal threshold methods presented in \cite{Massanet2012A,Massanet2013} as constructions that consist on an adequate scaling of the second or first variable of the two initial fuzzy implication functions. Thanks to their definition, the resulting operator has interesting properties like a controlled increasingness of the second or first variable, respectively. Also, the horizontal threshold method plays a crucial role in the characterization of $h$ and $(h,e)$-implications \cite{Fernandez-Peralta2022B,Massanet2012A}. Further, the horizontal threshold method was generalized in \cite{Yi2017} to consider a numerable collection of fuzzy implication functions. For the vertical threshold method, as far as we know, it does not exist such extension. In this section, we prove that the generalized $F$-chain-based construction method also generalizes these methods (in the finite case) by considering the weighted mean and an adequate scaling of the first or second variable.

Differently from the other methods, in this case we have to consider a preliminary step to the fuzzy implication functions considered in the generalized $F$-chain-based construction method in order to match the horizontal/vertical threshold methods. It is well-known that for any fuzzy implication function $I$ the following values on the boundary $I(x,1)$, $I(0,x)$ are fixed to 1 but the values $I(x,0)$, $I(1,x)$ depend on each operator. Further, if we respect the monotonicity and boundary conditions, we can replace those values with others. In accordance, in the next proposition we define the transformation which assigns $I(x,0)$, $I(1,x)$ values to zero except for the fixed value $I(1,1)$. Further, we prove that it preserves some of the properties considered in Section \ref{section:properties}.

\begin{proposition}\label{prop:zero_transformation} 
Let $I:[0,1]^2 \to [0,1]$ be a fuzzy implication function and let us define next binary functions
 $$   \IZD{I}(x,y) =
 \left\{ \begin{array}{ll}
    0 &   \text{if } x>0 \text{ and }y=0, \\
    I(x,y) &   \text{otherwise},
    \end{array}
    \right.
    \quad
    \IZU{I}(x,y) =
 \left\{ \begin{array}{ll}
    0 &   \text{if } x=1 \text{ and } y<1, \\
    I(x,y) &   \text{otherwise},
    \end{array}
    \right.
$$
$$
    \IZZ{I}(x,y) =
 \left\{ \begin{array}{ll}
    0 &   \text{if } (x=1 \text{ and } y<1) \text{ or } (x>0 \text{ and } y=0), \\ 
    I(x,y) &   \text{otherwise}.
    \end{array}
    \right.
$$
Then \IZD{\IZU{I}}=\IZU{\IZD{I}}=\IZZ{I} and \IZD{I}, \IZU{I} and \IZZ{I} are fuzzy implication functions. Further, the following statements hold:
\begin{enumerate}[(i)]
\item If $I$ satisfies \IP then $\IZD{I}$, $\IZU{I}$ and $\IZZ{I}$ satisfy \IP.
\item If $I$ satisfies \OP then $\IZD{I}$, $\IZU{I}$ and $\IZZ{I}$ satisfy \OP.
\item If $I$ satisfies \NP then \IZD{I} satisfies \NP too but \IZU{I} and \IZZ{I} do not satisfy \NP.
\item If $I$ satisfies \CB then $\IZD{I}$ satisfies \CB too but \IZU{I} and \IZZ{I} do not satisfy \CB.
\item If $I$ satisfies \LT then $\IZD{I}$, $\IZU{I}$ and $\IZZ{I}$ satisfy \LT.
\item \IZD{I}, \IZU{I} and \IZZ{I} do not satisfy \LF for any fuzzy implication function $I$.
\item If $I$ satisfies \CPN, \IZD{I}, \IZU{I} and \IZZ{I} do not necessarily satisfy \CPN.
\item If $I$ satisfies \LCPN, \IZD{I}, \IZU{I} and \IZZ{I} do not necessarily satisfy \LCPN.
\item If $I$ satisfies \RCPN, \IZD{I}, \IZU{I} and \IZZ{I} do not necessarily satisfy \RCPN.
\item If $I$ satisfies \PIT then $\IZD{I}$, $\IZU{I}$ and $\IZZ{I}$ satisfy \PIT.
\end{enumerate}
\end{proposition}

\begin{proof} ~
 \begin{itemize} 
 \item[(i)-(ii)] Straightforward, since the structure of the fuzzy implication function inside the square $(0,1)^2$ remains invariant.
 \item[(iii)] Let $I$ satisfy \NP, then
 $$\IZD{I}(1,y) = 
  \left\{ \begin{array}{ll}
    0 &   \text{if } y=0, \\
    I(1,y) &   \text{otherwise},
    \end{array} \right. = \left\{ \begin{array}{ll}
    0 &   \text{if } y=0, \\
    y &   \text{otherwise},
    \end{array} \right. =y.$$
However, \IZU{I} and \IZZ{I} do not satisfy \NP because $\IZU{I}(1,y)=\IZZ{I}(1,y)=0$ for all $y \in (0,1)$.
\item[(iv)] Let $I$ satisfy \CB, then
$$\IZD{I}(x,y)  =
\left\{ \begin{array}{ll}
    0 &   \text{if } x=1 \text{ and } y<1, \\
    I(x,y) &   \text{otherwise},
    \end{array}
    \right. \geq y.
$$
However, \IZU{I} and \IZZ{I} do not satisfy \CB because $\IZU{I}(1,y)=\IZZ{I}(1,y)=0 < y$ for all $y \in (0,1]$.
\item[(v)-(vi)] Straightforward.
\item[(vii)] Let $y=0$, $x \in (0,1)$ with $N(x) \not = 0$ and $I$ a fuzzy implication function such that $I(1,N(x)) \not = 0$ then
$$\IZD{I}(x,0)=0, \quad \IZD{I}(N(0),N(x)) = \IZD{I}(1,N(x)) = I(1,N(x)) \not = 0.$$
Let $x=1$, $y \in (0,1)$ with $N(y) \not = 1$ and $I$ a fuzzy implication function such that $I(N(y),0) \not = 0$ then
$$
\IZU{I}(1,y) = 0, \quad \IZU{I}(N(y),N(1)) = \IZU{I}(N(y),0) = I(N(y),0) \not = 0.
$$
\item[(viii)-(ix)] The proof is analogous to the one in Point (vii).
\item[(x)] Straightforward since in the property \PIT only points of $(0,1)^2$ are evaluated.
 \end{itemize}
\end{proof}

Having said this, we can prove that the horizontal/vertical threshold methods can be rewritten in terms of the generalized $F$-chain-based construction method. Let us consider the two methods separately.

\subsubsection{Horizontal threshold method}\label{subsubsection:horizontal_threshold}

In this case, we consider an increasing sequence in $(0,1)$ as the thresholds, the weighted mean with weights the difference between these thresholds and the input transformation of linear rescaling each interval between thresholds to $[0,1]$. As commented before, the transformation in Proposition \ref{prop:zero_transformation} is needed in order to ensure the equivalence between the two methods in this case. 

\begin{proposition}\label{prop:hor_method} 
Let $n \in \mathbb{N}$, $I_1,\ldots,I_n\in \mathcal{I}$ and $\{e_i\}^n_{i=0}$ an increasing sequence in $(0,1)$ with $e_0=0$ and $e_{n}=1$. Let us consider the weighted mean with weights $e_{i}-e_{i-1}$ for all $i \in \{1,\dots,n\}$, i.e.,
\begin{equation}
F(x_1,\dots,x_n) = \sum_{i=1}^n (e_i-e_{i-1})x_i,
\end{equation}
and the following functions
\begin{equation}\label{eq:ci}
c_{i}(x) = 
\left\{ \begin{array}{ll}
    0 &   \text{if } x \in [0,e_{i-1}], \\
    \frac{x-e_{i-1}}{e_{i}-e_{i-1}} &   \text{if } x \in (e_{i-1},e_{i}], \\
    1 & \text{otherwise},
    \end{array} \right.
\end{equation}
for all $i \in \{1,\dots,n\}$. In this case, the following equality holds:

\begin{align}
& F \left( I_1(x,c_1(y)),\bigoplus_{i=2}^n \IZD{I_i} (x,c_i(y))\right) 
 \nonumber\\
 &=     \left\{ \begin{array}{ll}
    1 & \text{if } x=0 \text{ or } y=1,\\
    e_1I_1(x,0) & \text{if } x>0 \text{ and } y=0,\\
    e_{j-1} + (e_{j}-e_{j-1})I_{j}\left(x,\frac{y-e_{j-1}}{e_{j}-e_{j-1}}\right) &   \text{if } x > 0 \text{ and } e_{j-1} < y \leq e_{j}. \\
    \end{array}
    \right. \label{eq:threshold1}
\end{align}
\end{proposition}

\begin{proof} 
\begin{itemize} 
\item If $x=0$ or $y=1$ then 
$$F \left( I_1(x,c_1(y)),\bigoplus_{i=2}^n \IZD{I_i} (x,c_i(y))\right)=1,$$
because we know by Theorem \ref{th:def_method} that \Is is a fuzzy implication function.
\item If $x>0$ and $y=0$ then
\begin{eqnarray*}
F \left( I_1(x,c_1(y)),\bigoplus_{i=2}^n \IZD{I_i} (x,c_i(y))\right) &=& e_1I_1(x,c_1(0)) + \sum_{i=2}^{n}(e_i-e_{i-1})\IZD{I_i} (x,c_i(0))\\
&=& e_1I_1(x,0) + \sum_{i=2}^{n}(e_i-e_{i-1})\IZD{I_i} (x,0) \\
&=& e_1I_1(x,0).
\end{eqnarray*}
\item If $x>0$ and $e_{j-1} < y < e_j$ with $j \in \{1,\dots,n-1\}$ then
\begin{eqnarray*}
F \left( I_1(x,c_1(y)),\bigoplus_{i=2}^n \IZD{I_i} (x,c_i(y))\right) &=& \sum_{i=1}^{j-1} (e_i-e_{i-1}) I_i (x,1) + (e_j-e_{j-1}) I_j \left(x,\frac{y-e_{j-1}}{e_j-e_{j-1}}\right)\\
& &+ \sum_{i=j+1}^{n} (e_i-e_{i-1})\IZD{I_i} (x,0) \\
&=& \sum_{i=1}^{j-1} (e_i-e_{i-1}) + (e_j-e_{j-1}) I_j \left(x,\frac{y-e_{j-1}}{e_j-e_{j-1}}\right) \\
&=& e_{j-1} + (e_j-e_{j-1}) I_j \left(x,\frac{y-e_{j-1}}{e_j-e_{j-1}}\right).
\end{eqnarray*}
\item If $x>0$ and $y= e_{j}$ with $j \in \{1,\dots,n-1\}$ then
\begin{eqnarray*}
F \left( I_1(x,c_1(y)),\bigoplus_{i=2}^n \IZD{I_i} (x,c_i(y))\right) &=& \sum_{i=1}^{j} (e_i-e_{i-1})I_i(x,1) + \sum_{i=j+1}^{n} (e_i-e_{i-1})\IZD{I_i} (x,0) \\
&=& \sum_{i=1}^{j} (e_i-e_{i-1}) = e_j.
\end{eqnarray*}
\item If $x>0$ and $e_{n-1} < y \leq 1$ then
\begin{eqnarray*}
F \left( I_1(x,c_1(y)),\bigoplus_{i=2}^n \IZD{I_i} (x,c_i(y))\right) &=& \sum_{i=1}^{n-1} (e_i-e_{i-1}) I_i (x,1) + (1-e_{n-1}) I_n \left(x,\frac{y-e_{n-1}}{1-e_{n-1}}\right)\\
&=& e_{n-1} + (1-e_{n-1}) I_n \left(x,\frac{y-e_{n-1}}{1-e_{n-1}}\right).
\end{eqnarray*}
\end{itemize}
\end{proof}

\begin{figure}
\end{figure}

 \begin{figure}[t]
\centering
\includegraphics[scale=0.1]{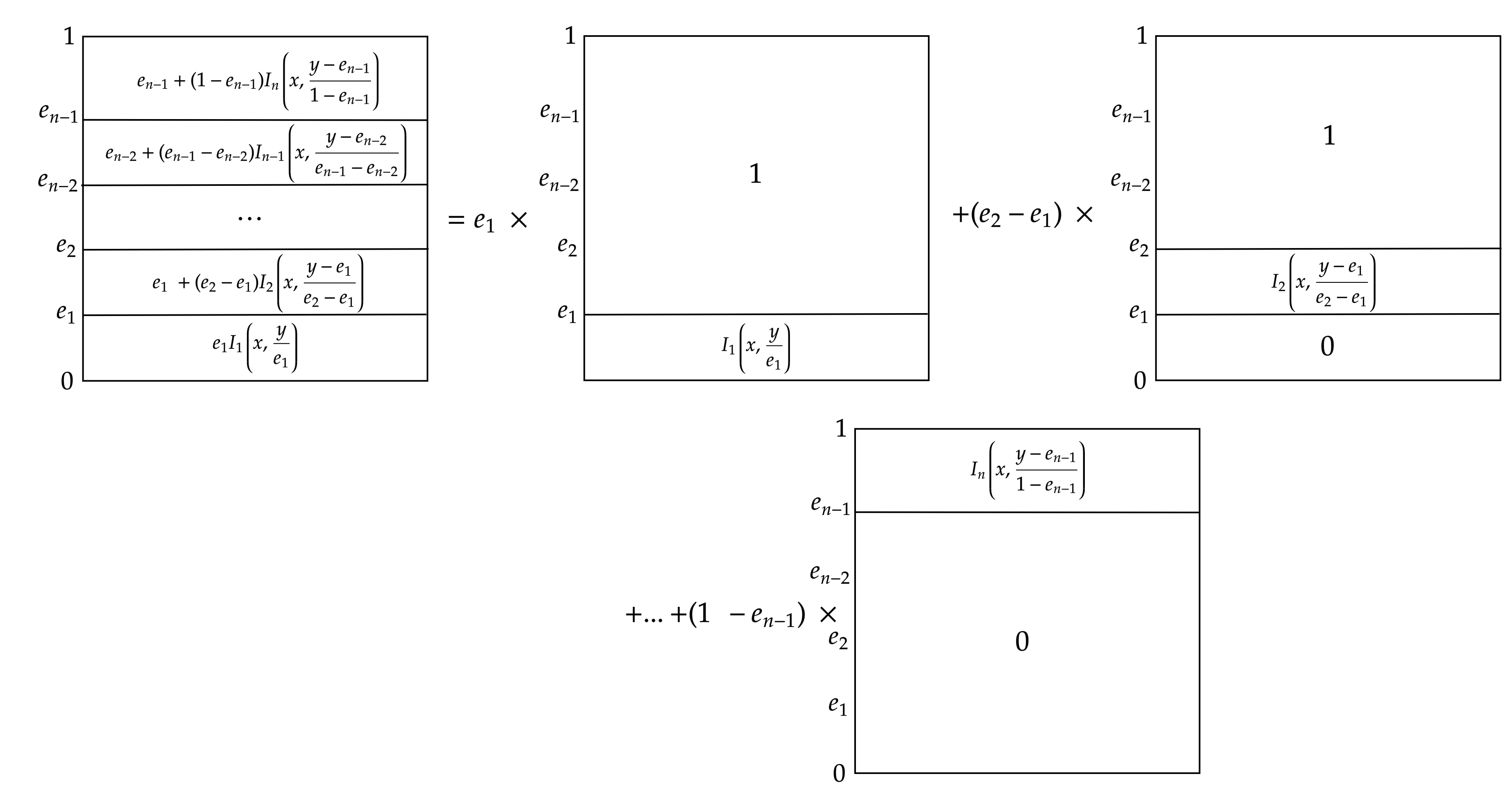}
\caption{Schematic representation of the generalized horizontal threshold method obtained through the generalized $F$-chain-based construction method.}\label{fig:horizontal_threshold}
 \end{figure}

In Figure \ref{fig:horizontal_threshold} there is the an sketch of how the generalized $F$-chain based construction is used to obtain the generalized horizontal threshold method. Notice that for obtaining a result such as Proposition \ref{prop:hor_method} it was necessary that our construction method includes different input transformation for each variable. Indeed, in the horizontal threshold construction method the first variable of the fuzzy implication functions remains invariant.

Now, we compare again the results obtain in Section \ref{section:properties} with the capability of the horizontal threshold method to conserve the additional properties of the input fuzzy implication functions.  To do so, we need to check if the used aggregation function $F$ and input transformation $\bm{c}$ satisfy some of the sufficient conditions.

First, we point out that $\bm{c}$ is an $F$-chain. Indeed, let $t \in [e_{j-1},e_j]$, then 
 \begin{eqnarray*}
F(\bm{c}(t)) &=&  \sum_{i=1}^{j-1} (e_i-e_{i-1}) \cdot 1 + (e_j-e_{j-1}) \cdot \frac{(t-e_{j-1})}{(e_j-e_{j-1})} + \sum_{i=j+1}^{n} (e_i-e_{i-1}) \cdot 0 \\ &=& e_{j-1} + t - e_{j-1} = t.
\end{eqnarray*}
Thus, according to Proposition \ref{prop:consequent_boundary} the horizontal threshold method preserves \NP and \CB properties. For the rest of results of Section \ref{section:properties} the conditions are not fulfilled. However, this is coherent with the fact that the horizontal threshold construction method does not satisfy properties like \OP or \IP. Let us highlight that to study properties like the contrapositivisations in these kind of constructions is a complex problem (see \cite[Proposition 23]{Massanet2012A}), so it was to be expected that our general results cannot be applied for all properties on specific methods.

\subsubsection{Vertical threshold method}

 Now, we consider the vertical threshold method. In this case, there exists only the definition using two fuzzy implication functions (see \cite[Definition 4]{Massanet2013}). However, thanks to our new technique we can straightforwardly generalize this construction method for a collection of fuzzy implication functions. Nonetheless, differently from the horizontal threshold method we need to use a decreasing sequence instead of an increasing one due to the decreasingness with respect to the first variable inherent of fuzzy implication functions.

\begin{proposition}\label{prop:ver_method} 
Let $\{e_i\}^n_{i=0}$ be an increasing sequence in $(0,1)$ with $e_0=0$ and $e_{n}=1$ and $\{\theta_i\}_{i=0}^n$ a decreasing sequence in $(0,1)$ with $\theta_0=1$ and $\theta_n=0$. Let us consider $F$ as the weighted mean with weights $\theta_{i-1}-\theta_{i}$ for all $i \in \{1,\dots,n\}$ and $c_i$ as in Eq. (\ref{eq:ci}) for all $i \in \{1,\dots,n\}$. In this case, the following equality holds:
 \begin{align}\label{eq:threshold2}
 &F \left(\bigoplus_{i=1}^{n-1} \IZU{I_i} (c_i(x),y),I_n(c_n(x),y)\right) \nonumber\\
 &=     \left\{ \begin{array}{ll}
    1 & \text{if } x=0 \text{ or } y=1,\\
    \theta_{n-1} I_n(1,y) & \text{if } x=1 \text{ and } y < 1, \\
    \theta_{j} + (\theta_{j-1}-\theta_j)I_j \left(\frac{x-e_{j-1}}{e_j - e_{j-1}},y\right)&   \text{if } e_{j-1} \leq  x < e_{j} \text{ and } y<1.
    \end{array} \right.
\end{align}
\end{proposition}

\begin{proof}
\begin{itemize}
\item If $x=0$ or $y=1$ then
$$F \left(\bigoplus_{i=1}^{n-1} \IZU{I_i} (c_i(x),y),I_n(c_n(x),y)\right)=1,$$
because we know by Theorem \ref{th:def_method} that \Is is a fuzzy implication function.
\item If $e_{j-1} < x < e_j$ and $y<1$ with $j \in \{1,\dots,n-1\}$ then
\begin{eqnarray*}
F \left(\bigoplus_{i=1}^{n-1} \IZU{I_i} (c_i(x),y),I_n(c_n(x),y)\right) &=& \sum_{i=1}^{j-1} (\theta_{i-1}-\theta_{i}) \IZU{I_i} (1,y) + (\theta_{j-1}-\theta_j) I_j \left(\frac{x-e_{j-1}}{e_j-e_{j-1}},y\right)\\
& &+ \sum_{i=j+1}^{n} (\theta_{i-1}-\theta_i)I_i (0,y) \\
&=& \sum_{i=j+1}^{n} (\theta_{i-1}-\theta_i) + (\theta_{j-1}-\theta_j) I_j \left(\frac{x-e_{j-1}}{e_j-e_{j-1}},y\right) \\
&=& \theta_j + (\theta_{j-1}-\theta_j) I_j \left(\frac{x-e_{j-1}}{e_j-e_{j-1}},y\right).
\end{eqnarray*}
\item If $x= e_j$ and $y<1$ with $j \in \{1,\dots,n-1\}$ then
\begin{eqnarray*}
F \left(\bigoplus_{i=1}^{n-1} \IZU{I_i} (c_i(x),y),I_n(c_n(x),y)\right) &=& \sum_{i=1}^{j} (\theta_{i-1}-\theta_i) \IZU{I_i} (1,y) +\sum_{i=j+1}^{n} (\theta_{i-1}-\theta_i)I_i (0,x) \\
&=& \sum_{i=j+1}^{n} (\theta_{i-1}-\theta_i) = \theta_j.
\end{eqnarray*}
\item If $e_{n-1} < x \leq 1$ and $y<1$ then
\begin{eqnarray*}
F \left(\bigoplus_{i=1}^{n-1} \IZU{I_i} (c_i(x),y),I_n(c_n(x),y)\right) &=& \sum_{i=1}^{n-1} (\theta_{i-1}-\theta_i) \IZU{I_i} (1,y) + (\theta_{n-1}-0) I_n \left(\frac{x-e_{n-1}}{1-e_{n-1}},y\right)\\
&=& \theta_{n-1} I_n \left(\frac{x-e_{n-1}}{1-e_{n-1}},y\right).
\end{eqnarray*}
\end{itemize}
\end{proof}

Similarly to the horizontal threshold construction method, the resulting fuzzy implication function has a controlled increasingness with respect to the second variable fixed by the sequence of thresholds.

\begin{corollary} Let $I$ be the fuzzy implication function in Eq. (\ref{eq:threshold2}), then
$$ I(e_i,y)=\theta_i, \quad \text{for all } i \in \{1,\dots,n-1\}.$$
\end{corollary}

Differently from the horizontal threshold construction method, the condition of $F$-chain also depends on the decreasing sequence $\{\theta_i\}_{i=0}^n$. However, we may select $\theta_i = 1-e_i$ for all $i \in \{1,\dots,n\}$ and we have
$$ F(x_1,\dots,x_n) = \sum_{i=1}^n (\theta_{i-1}-\theta_i)x_i = \sum_{i=1}^n (\theta_{i-1}-\theta_i)x_i = \sum_{i=1}^n (e_i - e_{i-1})x_i.$$
Thus, $\bm{c}$ is an $F$-chain as we proved in Section \ref{subsubsection:horizontal_threshold} and the  fuzzy implication function in Eq. (\ref{eq:threshold2}) satisfies \NP and \CB.


\subsection{Example of an ordinal sum}\label{subsection:ordinal_sum}

Another widely studied construction method is the ordinal sum of fuzzy implication functions, which generates a new operator that behaves as distinct fuzzy implication functions over a collection of disjoint subsquares. Unlike the case for continuous t-norms and t-conorms, however, there are multiple competing proposals for such ordinal sums \cite{Zhou2021B}. In this section, we demonstrate that the generalized $F$-chain-based construction method can also produce ordinal sums. Nevertheless, a comprehensive exploration of the relationship between this method and existing ordinal sum approaches remains an open question, which we leave as future work.


\begin{example}\label{example:ordinal_sum} 
Let $\{e_i\}^n_{i=0}$ be an increasing sequence in $(0,1)$ with $e_0=0$ and $e_{n}=1$. Let us consider $F$ as the weighted mean with weights $e_{i}-e_{i-1}$ for all $i \in \{1,\dots,n\}$ and $c_i$ as in Eq. (\ref{eq:ci}) for all $i \in \{1,\dots,n\}$. In this case, the following equality holds:

\begin{align}\label{eq:ordinal_sum}
&(I_1,\dots,I_n)_{F,\bm{c},\bm{c}}  \nonumber \\
&
 =     \left\{ \begin{array}{ll}
    1 & \text{if }x=0 \text{ or } y=1,\\
    1 &\text{if }x \in (e_{i-1},e_i], y \in [e_{j-1},e_j), i<j, \\
    1+e_{i-1}-e_i+(e_i-e_{i-1}) I_i \left(\frac{x-e_{i-1}}{e_i-e_{i-1}}, \frac{y-e_{i-1}}{e_i-e_{i-1}}\right) &\text{if }x \in (e_{i-1},e_i], y \in [e_{j-1},e_j), i=j, \\
    1+e_{j-1}-e_i+(e_j-e_{j-1})I_j \left(1,\frac{y-e_{j-1}}{e_j-e_{j-1}}\right) &\\
    + (e_i-e_{i-1})I_i\left(\frac{x-e_{i-1}}{e_i-e_{i-1}},0\right) &\text{if }x \in (e_{i-1},e_i], y \in [e_{j-1},e_j), i>j. \\
    \end{array}
    \right.
\end{align}
The proof is similar as the one in Propositions \ref{prop:hor_method} and \ref{prop:ver_method}. As an illustrative example, if we select $n=3$, $\{e_0,e_1,e_2,e_3\} = \{0,0.5,0.75,1\}$ and $I_1,I_2,I_3$ as in Example \ref{example:negation} (i.e.,  \L ukasiewicz, Reichenbach and Kleene-Dienes) in Eq. (\ref{eq:ordinal_sum}) we obtain the fuzzy implication function that can be visualized in Figure \ref{figure:ordinal_sum}.
\begin{figure}[h]
\centering
\includegraphics[scale=0.3]{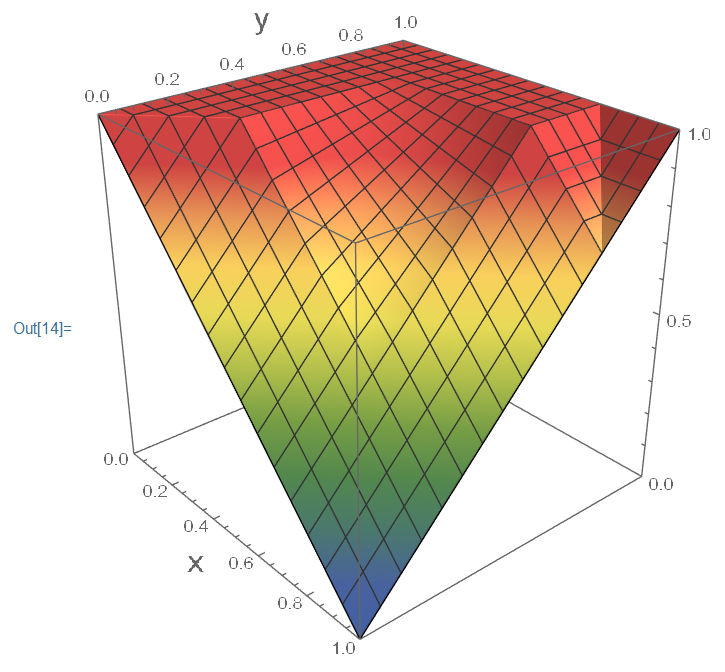}
\caption{Plot of the fuzzy implication function in Example \ref{example:ordinal_sum}.}\label{figure:ordinal_sum}
\end{figure}
\end{example}

\section{Conclusions}\label{section:conclusions}

In this paper, we introduced the generalized $F$-chain-based construction method, an extension of the original framework proposed by Mesiar et al. Our key contributions are threefold:

\begin{itemize}
\item \textit{Generalization of the $F$-chain based construction:} We proved that the original method can be extended to arbitrary collections of fuzzy implication functions and input transformations, rather than being restricted to $F$-chains.
\item \textit{Property preservation analysis}: We conducted a thorough investigation of how the construction preserves several essential properties. We provided illustrative examples of how our construction method can construct fuzzy implication functions that fulfill any of the considered properties.
\item \textit{Unification of existing methods}: We have shown that the generalized $F$-chain based construction serves as a unifying framework for several existing construction methods in the literature. By reformulating diverse approaches using specific aggregation functions and input transformations, we revealed novel connections between seemingly unrelated construction methods, offering a broader theoretical perspective.
\end{itemize}

The reported results show that the generalized $F$-chain-based construction method is sufficiently general to encompass existing methods while still allowing us to ensure the preservation of different additional properties. This may simplify further studies on fuzzy implication functions construction methods. Indeed, we believe that our compact description has multiple benefits, for instance an easier and compact algorithmic implementation of different families or the automatic verification of additional properties.

As future work, we highlight the search of necessity conditions for the considered properties, the study of other properties like the exchange principle, the law of importation or the $T$-conditionality, and the further exploration of the relation between the generalized $F$-chain based construction and other methods like ordinal sums.

\backmatter





\bmhead{Acknowledgements}

Raquel Fernandez-Peralta was funded by the EU NextGenerationEU through the Recovery and Resilience Plan for Slovakia under the project No. 09I03-03-V04-00557. Juan Vicente Riera was supported by the R+D+i Project PID2020-113870GB-I00-“Desarrollo de herramientas de Soft Computing para la Ayuda al Diagnóstico Clínico y a la Gestión de Emergencias (HESOCODICE)”, funded by MICIU /AEI/10.13039/501100011033/.

\bmhead{Data availability} Data sharing not applicable to this article as no datasets were generated or analyzed during
the current study.

\section*{Declarations}

\bmhead{Conflict of interest} The authors declare no conflict of interest









\bibliography{sn-article}


\begin{thebibliography}{38}
\ifx \bisbn   \undefined \def \bisbn  #1{ISBN #1}\fi
\ifx \binits  \undefined \def \binits#1{#1}\fi
\ifx \bauthor  \undefined \def \bauthor#1{#1}\fi
\ifx \batitle  \undefined \def \batitle#1{#1}\fi
\ifx \bjtitle  \undefined \def \bjtitle#1{#1}\fi
\ifx \bvolume  \undefined \def \bvolume#1{\textbf{#1}}\fi
\ifx \byear  \undefined \def \byear#1{#1}\fi
\ifx \bissue  \undefined \def \bissue#1{#1}\fi
\ifx \bfpage  \undefined \def \bfpage#1{#1}\fi
\ifx \blpage  \undefined \def \blpage #1{#1}\fi
\ifx \burl  \undefined \def \burl#1{\textsf{#1}}\fi
\ifx \doiurl  \undefined \def \doiurl#1{\url{https://doi.org/#1}}\fi
\ifx \betal  \undefined \def \betal{\textit{et al.}}\fi
\ifx \binstitute  \undefined \def \binstitute#1{#1}\fi
\ifx \binstitutionaled  \undefined \def \binstitutionaled#1{#1}\fi
\ifx \bctitle  \undefined \def \bctitle#1{#1}\fi
\ifx \beditor  \undefined \def \beditor#1{#1}\fi
\ifx \bpublisher  \undefined \def \bpublisher#1{#1}\fi
\ifx \bbtitle  \undefined \def \bbtitle#1{#1}\fi
\ifx \bedition  \undefined \def \bedition#1{#1}\fi
\ifx \bseriesno  \undefined \def \bseriesno#1{#1}\fi
\ifx \blocation  \undefined \def \blocation#1{#1}\fi
\ifx \bsertitle  \undefined \def \bsertitle#1{#1}\fi
\ifx \bsnm \undefined \def \bsnm#1{#1}\fi
\ifx \bsuffix \undefined \def \bsuffix#1{#1}\fi
\ifx \bparticle \undefined \def \bparticle#1{#1}\fi
\ifx \barticle \undefined \def \barticle#1{#1}\fi
\bibcommenthead
\ifx \bconfdate \undefined \def \bconfdate #1{#1}\fi
\ifx \botherref \undefined \def \botherref #1{#1}\fi
\ifx \url \undefined \def \url#1{\textsf{#1}}\fi
\ifx \bchapter \undefined \def \bchapter#1{#1}\fi
\ifx \bbook \undefined \def \bbook#1{#1}\fi
\ifx \bcomment \undefined \def \bcomment#1{#1}\fi
\ifx \oauthor \undefined \def \oauthor#1{#1}\fi
\ifx \citeauthoryear \undefined \def \citeauthoryear#1{#1}\fi
\ifx \endbibitem  \undefined \def \endbibitem {}\fi
\ifx \bconflocation  \undefined \def \bconflocation#1{#1}\fi
\ifx \arxivurl  \undefined \def \arxivurl#1{\textsf{#1}}\fi
\csname PreBibitemsHook\endcsname

\bibitem[\protect\citeauthoryear{Baczy\'nski and Jayaram}{2008}]{Baczynski2008}
\begin{bbook}
\bauthor{\bsnm{Baczy\'nski}, \binits{M.}},
\bauthor{\bsnm{Jayaram}, \binits{B.}}:
\bbtitle{{F}uzzy {I}mplications}.
\bsertitle{Studies in Fuzziness and Soft Computing},
vol. \bseriesno{231}.
\bpublisher{Springer},
\blocation{Berlin, Heidelberg}
(\byear{2008})
\end{bbook}
\endbibitem

\bibitem[\protect\citeauthoryear{Beliakov et~al.}{2010}]{Beliakov2010}
\begin{bbook}
\bauthor{\bsnm{Beliakov}, \binits{G.}},
\bauthor{\bsnm{Pradera}, \binits{A.}},
\bauthor{\bsnm{Calvo}, \binits{T.}}:
\bbtitle{{A}ggregation {F}unctions: {A} {G}uide for {P}ractitioners}.
\bsertitle{Studies in Fuzziness and Soft Computing},
vol. \bseriesno{221}.
\bpublisher{Springer},
\blocation{Berlin, Heidelberg}
(\byear{2010})
\end{bbook}
\endbibitem

\bibitem[\protect\citeauthoryear{Cao and Hu}{2022}]{Cao2022}
\begin{barticle}
\bauthor{\bsnm{Cao}, \binits{M.}},
\bauthor{\bsnm{Hu}, \binits{B.Q.}}:
\batitle{On the ordinal sum of fuzzy implications: {N}ew results and the
  distributivity over a class of overlap and grouping functions}.
\bjtitle{Fuzzy Sets and Systems}
\bvolume{446},
\bfpage{93}--\blpage{123}
(\byear{2022})
\end{barticle}
\endbibitem

\bibitem[\protect\citeauthoryear{Calvo et~al.}{2002}]{Calvo2002}
\begin{bbook}
\bauthor{\bsnm{Calvo}, \binits{T.}},
\bauthor{\bsnm{Mayor}, \binits{G.}},
\bauthor{\bsnm{Mesiar}, \binits{R.}}:
\bbtitle{{A}ggregation {O}perators}.
\bsertitle{Studies in Fuzziness and Soft Computing},
vol. \bseriesno{97}.
\bpublisher{Springer},
\blocation{Verlag Berlin Heidelberg}
(\byear{2002})
\end{bbook}
\endbibitem

\bibitem[\protect\citeauthoryear{Dombi and Baczy\'nski}{2020}]{Baczynski2020B}
\begin{barticle}
\bauthor{\bsnm{Dombi}, \binits{J.}},
\bauthor{\bsnm{Baczy\'nski}, \binits{M.}}:
\batitle{{General Characterization of Implication's Distributivity Properties:
  The Preference Implication}}.
\bjtitle{IEEE Transactions on Fuzzy Systems}
\bvolume{28}(\bissue{11}),
\bfpage{2982}--\blpage{2995}
(\byear{2020})
\end{barticle}
\endbibitem

\bibitem[\protect\citeauthoryear{Fernandez-Peralta}{2025a}]{Fernandez-Peralta2025}
\begin{botherref}
\oauthor{\bsnm{Fernandez-Peralta}, \binits{R.}}:
A Comprehensive Survey of Fuzzy Implication Functions
(2025).
\url{https://arxiv.org/abs/2503.05702}
\end{botherref}
\endbibitem

\bibitem[\protect\citeauthoryear{Fernandez-Peralta}{2025b}]{Fernandez-Peralta2025B}
\begin{botherref}
\oauthor{\bsnm{Fernandez-Peralta}, \binits{R.}}:
Fuzzy Implicative Rules: A Unified Approach
(2025).
\url{https://arxiv.org/abs/2504.03000}
\end{botherref}
\endbibitem

\bibitem[\protect\citeauthoryear{Fernandez-Peralta
  et~al.}{2023}]{fernandez2023subgroup}
\begin{bchapter}
\bauthor{\bsnm{Fernandez-Peralta}, \binits{R.}},
\bauthor{\bsnm{Massanet}, \binits{S.}},
\bauthor{\bsnm{Gupta}, \binits{M.}},
\bauthor{\bsnm{Nanavati}, \binits{K.}},
\bauthor{\bsnm{Jayaram}, \binits{B.}}:
\bctitle{Subgroup discovery through sharp transitions using implicative type
  rules}.
In: \bbtitle{2023 IEEE International Conference on Fuzzy Systems (FUZZ)},
pp. \bfpage{1}--\blpage{6}
(\byear{2023}).
\bcomment{IEEE}
\end{bchapter}
\endbibitem

\bibitem[\protect\citeauthoryear{Fernandez-Peralta
  et~al.}{2021}]{Fernandez-Peralta2021}
\begin{barticle}
\bauthor{\bsnm{Fernandez-Peralta}, \binits{R.}},
\bauthor{\bsnm{Massanet}, \binits{S.}},
\bauthor{\bsnm{Mir}, \binits{A.}}:
\batitle{{On strict T-power invariant implications: Properties and
  intersections}}.
\bjtitle{Fuzzy Sets and Systems}
\bvolume{423},
\bfpage{1}--\blpage{28}
(\byear{2021})
\end{barticle}
\endbibitem

\bibitem[\protect\citeauthoryear{Fernandez-Peralta
  et~al.}{2022}]{Fernandez-Peralta2022B}
\begin{barticle}
\bauthor{\bsnm{Fernandez-Peralta}, \binits{R.}},
\bauthor{\bsnm{Massanet}, \binits{S.}},
\bauthor{\bsnm{Mir}, \binits{A.}}:
\batitle{{Characterization of generalized {$(h,e)$}-implications based on the
  characterization of {$(f,e)$} and {$(g,e)$}-implications}}.
\bjtitle{Information Sciences}
\bvolume{612},
\bfpage{1145}--\blpage{1170}
(\byear{2022})
\end{barticle}
\endbibitem

\bibitem[\protect\citeauthoryear{Fodor and Roubens}{1994}]{Fodor1994}
\begin{bbook}
\bauthor{\bsnm{Fodor}, \binits{J.}},
\bauthor{\bsnm{Roubens}, \binits{M.}}:
\bbtitle{{F}uzzy {P}reference {M}odelling and {M}ulticriteria {D}ecision
  {S}upport}.
\bpublisher{Kluwer Academic Publishers},
\blocation{Dordrecht}
(\byear{1994})
\end{bbook}
\endbibitem

\bibitem[\protect\citeauthoryear{González-Hidalgo et~al.}{2018}]{Gonzalez2018}
\begin{barticle}
\bauthor{\bsnm{González-Hidalgo}, \binits{M.}},
\bauthor{\bsnm{Massanet}, \binits{S.}},
\bauthor{\bsnm{Mir}, \binits{A.}},
\bauthor{\bsnm{Ruiz-Aguilera}, \binits{D.}}:
\batitle{Improving salt and pepper noise removal using a fuzzy mathematical
  morphology-based filter}.
\bjtitle{Applied Soft Computing}
\bvolume{63},
\bfpage{167}--\blpage{180}
(\byear{2018})
\end{barticle}
\endbibitem

\bibitem[\protect\citeauthoryear{García-Lapresta and {Marques
  Pereira}}{2008}]{GARCIALAPRESTA2008}
\begin{barticle}
\bauthor{\bsnm{García-Lapresta}, \binits{J.L.}},
\bauthor{\bsnm{{Marques Pereira}}, \binits{R.A.}}:
\batitle{The self-dual core and the anti-self-dual remainder of an aggregation
  operator}.
\bjtitle{Fuzzy Sets and Systems}
\bvolume{159}(\bissue{1}),
\bfpage{47}--\blpage{62}
(\byear{2008})
\end{barticle}
\endbibitem

\bibitem[\protect\citeauthoryear{Grabisch et~al.}{2009}]{Grabisch2009}
\begin{bbook}
\bauthor{\bsnm{Grabisch}, \binits{M.}},
\bauthor{\bsnm{Marichal}, \binits{J.-L.}},
\bauthor{\bsnm{Mesiar}, \binits{R.}},
\bauthor{\bsnm{Pap}, \binits{E.}}:
\bbtitle{{A}ggregation {F}unctions}.
\bsertitle{Encyclopedia of Mathematics and its Applications},
vol. \bseriesno{127}.
\bpublisher{Cambridge University Press},
\blocation{Cambridge}
(\byear{2009})
\end{bbook}
\endbibitem

\bibitem[\protect\citeauthoryear{Goguen}{1967}]{Goguen1967}
\begin{barticle}
\bauthor{\bsnm{Goguen}, \binits{J.A.}}:
\batitle{{L-fuzzy sets}}.
\bjtitle{Journal of Mathematical Analysis and Applications}
\bvolume{18}(\bissue{1}),
\bfpage{145}--\blpage{174}
(\byear{1967})
\end{barticle}
\endbibitem

\bibitem[\protect\citeauthoryear{Jayaram}{2006}]{Jayaram2006}
\begin{barticle}
\bauthor{\bsnm{Jayaram}, \binits{B.}}:
\batitle{{Contrapositive symmetrisation of fuzzy implications-Revisited}}.
\bjtitle{Fuzzy Sets and Systems}
\bvolume{157}(\bissue{17}),
\bfpage{2291}--\blpage{2310}
(\byear{2006})
\end{barticle}
\endbibitem

\bibitem[\protect\citeauthoryear{Jayaram}{2008a}]{Jayaram2008}
\begin{barticle}
\bauthor{\bsnm{Jayaram}, \binits{B.}}:
\batitle{{On the Law of Importation $(x \wedge y) \longrightarrow z \equiv (x
  \longrightarrow (y \longrightarrow z))$ in Fuzzy Logic}}.
\bjtitle{IEEE Transactions on Fuzzy Systems}
\bvolume{16},
\bfpage{130}--\blpage{144}
(\byear{2008})
\end{barticle}
\endbibitem

\bibitem[\protect\citeauthoryear{Jayaram}{2008b}]{Jayaram2008B}
\begin{barticle}
\bauthor{\bsnm{Jayaram}, \binits{B.}}:
\batitle{{Rule reduction for efficient inferencing in similarity based
  reasoning}}.
\bjtitle{International Journal of Approximate Reasoning}
\bvolume{48}(\bissue{1}),
\bfpage{156}--\blpage{173}
(\byear{2008})
\end{barticle}
\endbibitem

\bibitem[\protect\citeauthoryear{Jin et~al.}{2019}]{Jin2019}
\begin{barticle}
\bauthor{\bsnm{Jin}, \binits{L.}},
\bauthor{\bsnm{Mesiar}, \binits{R.}},
\bauthor{\bsnm{Kalina}, \binits{M.}},
\bauthor{\bsnm{Špirková}, \binits{J.}},
\bauthor{\bsnm{Borkotokey}, \binits{S.}}:
\batitle{Generalized phi-transformations of aggregation functions}.
\bjtitle{Fuzzy Sets and Systems}
\bvolume{372},
\bfpage{124}--\blpage{141}
(\byear{2019})
\end{barticle}
\endbibitem

\bibitem[\protect\citeauthoryear{Klement et~al.}{2000}]{Klement2000}
\begin{bbook}
\bauthor{\bsnm{Klement}, \binits{E.P.}},
\bauthor{\bsnm{Mesiar}, \binits{R.}},
\bauthor{\bsnm{Pap}, \binits{E.}}:
\bbtitle{Triangular Norms}.
\bsertitle{Trends in Logic},
vol. \bseriesno{8}.
\bpublisher{Kluwer Academic Publishers},
\blocation{Dordrecht}
(\byear{2000})
\end{bbook}
\endbibitem

\bibitem[\protect\citeauthoryear{Mendel}{2023}]{Mendel2023}
\begin{bbook}
\bauthor{\bsnm{Mendel}, \binits{J.M.}}:
\bbtitle{Explainable Uncertain Rule-Based Fuzzy Systems}.
\bpublisher{Springer},
\blocation{Cham}
(\byear{2023})
\end{bbook}
\endbibitem

\bibitem[\protect\citeauthoryear{Massanet et~al.}{2024}]{Massanet2024}
\begin{barticle}
\bauthor{\bsnm{Massanet}, \binits{S.}},
\bauthor{\bsnm{Fernandez-Peralta}, \binits{R.}},
\bauthor{\bsnm{Baczyński}, \binits{M.}},
\bauthor{\bsnm{Jayaram}, \binits{B.}}:
\batitle{On valuable and troubling practices in the research on classes of
  fuzzy implication functions}.
\bjtitle{Fuzzy Sets and Systems}
\bvolume{476},
\bfpage{108786}
(\byear{2024})
\end{barticle}
\endbibitem

\bibitem[\protect\citeauthoryear{Mesiar and
  Koles{\'a}rov{\'a}}{2019}]{Mesiar2019}
\begin{bchapter}
\bauthor{\bsnm{Mesiar}, \binits{R.}},
\bauthor{\bsnm{Koles{\'a}rov{\'a}}, \binits{A.}}:
\bctitle{{C}onstruction of {F}uzzy implication {F}unctions {B}ased on
  $f$-chains}.
In: \beditor{\bsnm{Hala{\v{s}}}, \binits{R.}},
\beditor{\bsnm{Gagolewski}, \binits{M.}},
\beditor{\bsnm{Mesiar}, \binits{R.}} (eds.)
\bbtitle{New Trends in Aggregation Theory. AGOP 2019}.
\bsertitle{Advances in Intelligent Systems and Computing},
vol. \bseriesno{981},
pp. \bfpage{315}--\blpage{326}.
\bpublisher{Springer},
\blocation{Cham}
(\byear{2019})
\end{bchapter}
\endbibitem

\bibitem[\protect\citeauthoryear{Mis et~al.}{2025}]{Mis2025}
\begin{botherref}
\oauthor{\bsnm{Mis}, \binits{K.}},
\oauthor{\bsnm{Kaczmarek-Majer}, \binits{K.}},
\oauthor{\bsnm{Baczyński}, \binits{M.}}:
Fuzzy linguistic summaries and the double negation property
(2025)
\doiurl{10.36227/techrxiv.173747538.86263707/v1}
\end{botherref}
\endbibitem

\bibitem[\protect\citeauthoryear{Munar et~al.}{2023}]{Munar2023}
\begin{botherref}
\oauthor{\bsnm{Munar}, \binits{M.}},
\oauthor{\bsnm{Massanet}, \binits{S.}},
\oauthor{\bsnm{Ruiz-Aguilera}, \binits{D.}}:
{A review on logical connectives defined on finite chains}.
Fuzzy Sets and Systems,
108469
(2023)
\end{botherref}
\endbibitem

\bibitem[\protect\citeauthoryear{Massanet et~al.}{2017}]{Massanet2017}
\begin{barticle}
\bauthor{\bsnm{Massanet}, \binits{S.}},
\bauthor{\bsnm{Recasens}, \binits{J.}},
\bauthor{\bsnm{Torrens}, \binits{J.}}:
\batitle{{Fuzzy implication functions based on powers of continuous t-norms}}.
\bjtitle{International Journal of Approximate Reasoning}
\bvolume{83},
\bfpage{265}--\blpage{279}
(\byear{2017})
\end{barticle}
\endbibitem

\bibitem[\protect\citeauthoryear{Massanet and Torrens}{2012}]{Massanet2012A}
\begin{barticle}
\bauthor{\bsnm{Massanet}, \binits{S.}},
\bauthor{\bsnm{Torrens}, \binits{J.}}:
\batitle{{Threshold generation method of construction of a new implication from
  two given ones}}.
\bjtitle{Fuzzy Sets and Systems}
\bvolume{205},
\bfpage{50}--\blpage{75}
(\byear{2012})
\end{barticle}
\endbibitem

\bibitem[\protect\citeauthoryear{Massanet and Torrens}{2013}]{Massanet2013}
\begin{barticle}
\bauthor{\bsnm{Massanet}, \binits{S.}},
\bauthor{\bsnm{Torrens}, \binits{J.}}:
\batitle{{On the vertical threshold generation method of fuzzy implication and
  its properties}}.
\bjtitle{Fuzzy Sets and Systems}
\bvolume{226},
\bfpage{32}--\blpage{52}
(\byear{2013})
\end{barticle}
\endbibitem

\bibitem[\protect\citeauthoryear{Nanavati et~al.}{2024}]{Nanavati2024}
\begin{barticle}
\bauthor{\bsnm{Nanavati}, \binits{K.}},
\bauthor{\bsnm{Gupta}, \binits{M.}},
\bauthor{\bsnm{Jayaram}, \binits{B.}}:
\batitle{Fuzzy implications - a (dis)similarity perspective}.
\bjtitle{International Journal of Approximate Reasoning}
\bvolume{168},
\bfpage{109145}
(\byear{2024})
\end{barticle}
\endbibitem

\bibitem[\protect\citeauthoryear{Reiser et~al.}{2013}]{Reiser2013}
\begin{barticle}
\bauthor{\bsnm{Reiser}, \binits{R.H.S.}},
\bauthor{\bsnm{Bedregal}, \binits{B.}},
\bauthor{\bsnm{Baczyński}, \binits{M.}}:
\batitle{{Aggregating fuzzy implications}}.
\bjtitle{Information Sciences}
\bvolume{253},
\bfpage{126}--\blpage{146}
(\byear{2013})
\end{barticle}
\endbibitem

\bibitem[\protect\citeauthoryear{Su et~al.}{2015}]{Su2015b}
\begin{barticle}
\bauthor{\bsnm{Su}, \binits{Y.}},
\bauthor{\bsnm{Xie}, \binits{A.}},
\bauthor{\bsnm{Liu}, \binits{H.-W.}}:
\batitle{{On ordinal sum implications}}.
\bjtitle{Information Sciences}
\bvolume{293},
\bfpage{251}--\blpage{262}
(\byear{2015})
\end{barticle}
\endbibitem

\bibitem[\protect\citeauthoryear{Su et~al.}{2016}]{Su2016}
\begin{barticle}
\bauthor{\bsnm{Su}, \binits{Y.}},
\bauthor{\bsnm{Zong}, \binits{W.}},
\bauthor{\bsnm{Liu}, \binits{H.-W.}}:
\batitle{Distributivity of the ordinal sum implications over t-norms and
  t-conorms}.
\bjtitle{IEEE Transactions on Fuzzy Systems}
\bvolume{24}(\bissue{4}),
\bfpage{827}--\blpage{840}
(\byear{2016})
\end{barticle}
\endbibitem

\bibitem[\protect\citeauthoryear{Trillas et~al.}{2008}]{Trillas2008}
\begin{barticle}
\bauthor{\bsnm{Trillas}, \binits{E.}},
\bauthor{\bsnm{Mas}, \binits{M.}},
\bauthor{\bsnm{Monserrat}, \binits{M.}},
\bauthor{\bsnm{Torrens}, \binits{J.}}:
\batitle{{On the representation of fuzzy rules}}.
\bjtitle{International Journal of Approximate Reasoning}
\bvolume{48}(\bissue{2}),
\bfpage{583}--\blpage{597}
(\byear{2008})
\end{barticle}
\endbibitem

\bibitem[\protect\citeauthoryear{Vemuri and Jayaram}{2014}]{Vemuri2014}
\begin{barticle}
\bauthor{\bsnm{Vemuri}, \binits{N.R.}},
\bauthor{\bsnm{Jayaram}, \binits{B.}}:
\batitle{{Representations through a monoid on the set of fuzzy implications}}.
\bjtitle{Fuzzy Sets and Systems}
\bvolume{247},
\bfpage{51}--\blpage{67}
(\byear{2014})
\end{barticle}
\endbibitem

\bibitem[\protect\citeauthoryear{Walker and Walker}{2002}]{Walker2002}
\begin{barticle}
\bauthor{\bsnm{Walker}, \binits{C.L.}},
\bauthor{\bsnm{Walker}, \binits{E.A.}}:
\batitle{{Powers of t-norms}}.
\bjtitle{Fuzzy Sets and Systems}
\bvolume{129}(\bissue{1}),
\bfpage{1}--\blpage{18}
(\byear{2002})
\end{barticle}
\endbibitem

\bibitem[\protect\citeauthoryear{Yi and Qin}{2017}]{Yi2017}
\begin{bchapter}
\bauthor{\bsnm{Yi}, \binits{Z.-H.}},
\bauthor{\bsnm{Qin}, \binits{F.}}:
\bctitle{{Extended Threshold Generation of a New Class of Fuzzy Implications}}.
In: \beditor{\bsnm{Fan}, \binits{T.-H.}},
\beditor{\bsnm{Chen}, \binits{S.-L.}},
\beditor{\bsnm{Wang}, \binits{S.-M.}},
\beditor{\bsnm{Li}, \binits{Y.-M.}} (eds.)
\bbtitle{Quantitative Logic and Soft Computing 2016}.
\bsertitle{Advances in Intelligent Systems and Computing},
vol. \bseriesno{510},
pp. \bfpage{281}--\blpage{291}.
\bpublisher{Springer},
\blocation{Cham}
(\byear{2017})
\end{bchapter}
\endbibitem

\bibitem[\protect\citeauthoryear{Yi et~al.}{2024}]{Yi2024}
\begin{barticle}
\bauthor{\bsnm{Yi}, \binits{Z.-H.}},
\bauthor{\bsnm{Yao}, \binits{L.-J.}},
\bauthor{\bsnm{Qin}, \binits{F.}}:
\batitle{{Some Properties of the Extended Threshold Generated Implications}}.
\bjtitle{Fuzzy Information and Engineering}
\bvolume{16}(\bissue{3}),
\bfpage{175}--\blpage{193}
(\byear{2024})
\end{barticle}
\endbibitem

\bibitem[\protect\citeauthoryear{Zhou}{2021}]{Zhou2021B}
\begin{barticle}
\bauthor{\bsnm{Zhou}, \binits{H.}}:
\batitle{Two general construction ways toward unified framework of ordinal sums
  of fuzzy implications}.
\bjtitle{IEEE Transactions on Fuzzy Systems}
\bvolume{29}(\bissue{4}),
\bfpage{846}--\blpage{860}
(\byear{2021})
\end{barticle}
\endbibitem

\end{thebibliography}

\end{document}